\def\TFAArxivBuild{1}
\documentclass{article}
\usepackage{iclr2027_conference,times}
\usepackage[T1]{fontenc}
\usepackage{amsmath,amssymb,amsthm,booktabs,graphicx,multirow}
\usepackage[table]{xcolor}
\usepackage{caption}
\usepackage{hyperref}
\usepackage[capitalize,nameinlink]{cleveref}

\newtheorem{proposition}{Proposition}
\graphicspath{{figures/}}
\ifdefined\TFAArxivBuild
\iclrfinalcopy
\title{Calibrating Retrieval Geometry: Reliability-Guided Training-Free Aggregation for Visual Place Recognition}
\author{\parbox{0.97\textwidth}{\centering\normalfont
Xin Li\textsuperscript{1,2,*},
Zhimin Mao\textsuperscript{1,2},
Shang Wang\textsuperscript{1,2},
Siyuan Duan\textsuperscript{1,2},
Geng Zhang\textsuperscript{1,*}\\[6pt]
\small
\textsuperscript{1}Key Laboratory of Spectral Imaging Technology CAS,\\
Xi'an Institute of Optics and Precision Mechanics,\\
Chinese Academy of Sciences, Xi'an, 710119, China\\[3pt]
\textsuperscript{2}University of Chinese Academy of Sciences,\\
Beijing, 100049, China\\[5pt]
\textsuperscript{*}Corresponding authors:
\texttt{lixin200@mails.ucas.edu.cn}, \texttt{gzhang@opt.ac.cn}
}}
\let\TFAPreprintMaketitle\maketitle
\renewcommand{\maketitle}{\TFAPreprintMaketitle\fancyhead{}\lhead{Preprint}}
\hypersetup{
 pdftitle={Calibrating Retrieval Geometry: Reliability-Guided Training-Free Aggregation for Visual Place Recognition},
 pdfauthor={Xin Li, Zhimin Mao, Shang Wang, Siyuan Duan, Geng Zhang}
}

\else
\title{TFA: Calibrating Frozen Retrieval Geometry without Labels}
\author{Anonymous authors}
\fi
\ifdefined\TFAReviewBuild\else
\renewcommand{\iclrruler}[1]{}
\fi
\begin{document}
\raggedbottom
\maketitle
\begin{abstract}
Frozen visual foundation models offer transferable representations for visual
place recognition, but fixed aggregation rules can suppress useful distinctions
when deployed in new environments. We introduce TFA, a reliability-guided,
training-free aggregation method that calibrates frozen representations without
place labels or task-specific weight updates. Our central observation is that
reproducible retrieval does not necessarily imply discriminative retrieval:
independently constructed codebooks can consistently select a small set of
database hubs. TFA combines cross-codebook agreement, retrieval coverage, and
spectral statistics to control residual assignment, spectral shaping, and
global-feature fusion. The database-only variant establishes its rules before
accessing deployment queries; TFA-C64 uses 64 disjoint unlabeled target images
to assess reference-to-query information retention and calibrate decisions for
subsequent queries. Its inner spectral kernel recovers the original descriptor
similarity exactly when intervention strengths vanish. Controlled comparisons
across 20 ground protocols use a fixed DINOv2-B backbone and identical input
resolution. Database-only TFA improves Recall@1 over AnyLoc by 17.39 percentage
points on MSLS-val and 9.55 points on SPED, while comparisons with TFA-C64
show the benefit of the target-aware controller in driving environments. Across
eight aerial/cross-view protocols with DINOv2 and DINOv3, database-only TFA
attains the highest Recall@1 among the compared training-free heads in 14 of
16 backbone--protocol combinations. In a separate native-system comparison,
the DINOv2-G-based 64-image calibration system achieves 91.46\% Recall@1 on
Pitts30k and 76.29\% on VPAIR, exceeding the displayed training-free comparators
across all five evaluated benchmarks. These results demonstrate that
reliability-guided aggregation can recover additional retrieval capability
from frozen representations, providing a practical baseline for new retrieval
environments with scarce place supervision.
\end{abstract}

\begin{figure}[!ht]
 \centering\includegraphics[width=\linewidth]{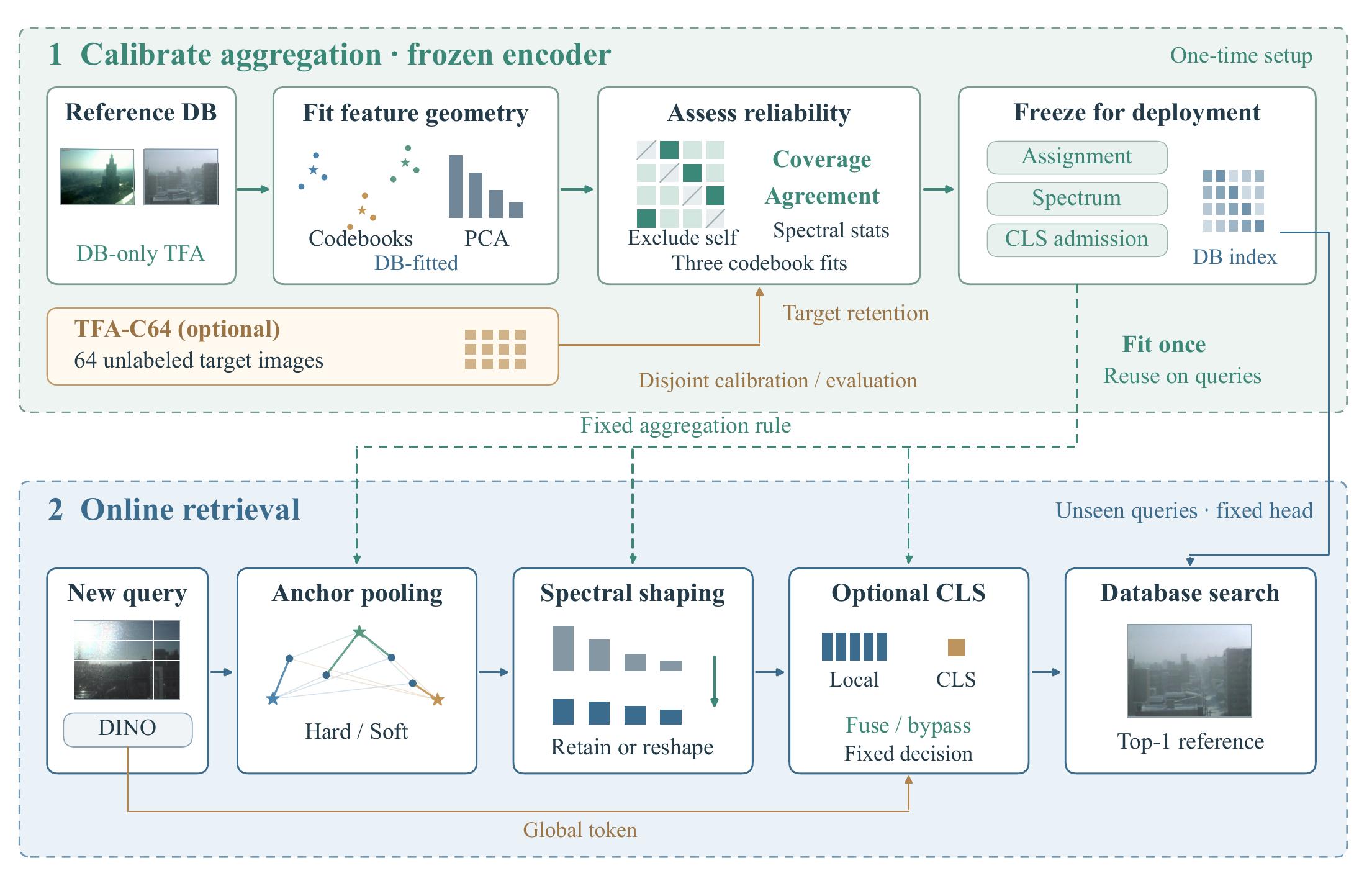}
 \caption{\textbf{TFA overview.} DB-fitted geometry and reliability evidence
 determine a fixed aggregation head. C64 adds 64 disjoint target images for
 calibration. Dashed arrows: fixed controls; solid: data.
 Feature diagrams are schematic; retrieval shows DB-only SPED q55, seed 0.}
 \label{fig:tfa-algorithm-pipeline}
\end{figure}

\section{Introduction}
Visual place recognition retrieves observations of the same location despite
changes in appearance, viewpoint, and sensing conditions. In a new deployment,
the reference database may be available well before representative
place-labeled training pairs. Frozen visual foundation models make this regime
increasingly practical: their dense features can be reused without task-specific
encoder training. AnyLoc demonstrated the cross-environment value of combining
these features with unsupervised VLAD~\citep{anyloc,dinov2}. This raises a
representation-use question: how much retrieval capability can be recovered
from fixed features by adapting their aggregation, without collecting positives
or updating model weights?

Aggregation determines which distinctions in the patch representation survive
as image-level similarity. A patch can contribute to one anchor or several;
dominant descriptor directions can be retained or downweighted; a global token
can supply complementary semantic evidence. A fixed choice need not transfer
between maps. Under the same frozen feature interface, full whitening improves
SPED~\citep{sped} from 76.00 to 85.34 R@1 but reduces
AmsterTime~\citep{amstertime} from 45.71 to 0.76.
Partial whitening is a substantially stronger general control, yet also leaves
map-dependent gains unrecovered. The challenge is therefore to decide which
transformation is supported by a new deployment using unlabeled evidence.

The central challenge is interpreting unlabeled evidence. Agreement
across independently constructed codebooks can reflect reproducible place
structure, or merely the repeated selection of a few database hubs, a known
high-dimensional nearest-neighbor phenomenon~\citep{radovanovic2010hubs}. We call
the latter failure \emph{stable collapse}. Broad retrieval support alone has
the complementary problem: noisy rankings can cover many database images.
Consequently, neither construction stability nor support is sufficient by
itself. Their joint behavior provides evidence for adapting aggregation, while
database-only measurements may still miss a database--query distribution shift.

We introduce TFA, a database-calibrated aggregation head. Reference images
serve as pseudo queries after self and descriptor-equivalent matches are
excluded. Support and cross-codebook agreement control assignment, spectral
carrier selection, and global-token admission. Split-database spectral
statistics determine the strengths of a reversible inner kernel. The
complete head is fixed before processing any deployment query; vocabulary
fitting and PCA use only the reference database.

Our contributions are a database-only reliability controller for frozen
retrieval features, a direct-sum spectral formulation with exact
original-geometry fallback at zero intervention, and a scene-stratified
evaluation of its transfer across representations. Under a common DINOv2-B
interface, TFA improves on all three training-free controls in nine ground
protocols and ties the best in one. Its gains over AnyLoc include 17.39 R@1
points on MSLS-val and 9.55 on SPED. The aerial comparison extends the study
to DINOv3, with fourteen of sixteen backbone--protocol comparisons favoring
TFA over the displayed training-free alternatives.

TFA-C64 extends database-only deployment with 64 disjoint unlabeled target
images; TFA-FullQ serves as an idealized full-query calibration reference.
A separate larger-backbone, dual-vocabulary
TFA-C64 system retains 97.4--99.8\% of its paired full-query R@1 across five
map units. These extensions measure the benefit of target observations
beyond the database-only deployment setting.

\section{Related work}
\paragraph{Learned place representations.}
NetVLAD makes residual aggregation differentiable and learns it with place
supervision~\citep{netvlad}. SALAD reformulates local-feature assignment as
optimal transport, including a dustbin for features left out of the aggregated
representation~\citep{salad}. BoQ instead learns global queries that aggregate
image features through cross-attention~\citep{boq}. SelaVPR++ adapts foundation
models with lightweight modules and combines binary initial retrieval with
floating-point global-feature reranking~\citep{selavpr}. These methods provide
the supervised reference systems in our comparisons. TFA addresses deployment
without task-specific weight optimization, adapting the aggregation rule
through unlabeled statistics rather than learned VPR parameters.

\paragraph{Training-free aggregation and segment retrieval.}
AnyLoc combines off-the-shelf self-supervised features with unsupervised
aggregation, establishing their utility across structured and unstructured
environments~\citep{anyloc}. Its VLAD formulation is the residual-aggregation
reference in our experiments. TF-VPR benchmarks frozen foundation features
and uses training-free graph-attention and cross-attention modules to aggregate
token relationships for place retrieval~\citep{tfvpr}.
Riemannian Invariant Aggregation (RIA)
models second-order patch statistics as covariance descriptors and maps their
positive-definite geometry into a Euclidean representation; it reports both
zero-shot and fine-tuned configurations~\citep{ria}. Our training-free
comparison concerns its zero-shot formulation. Revisit Anything introduces
SegVLAD, which encodes image segments and neighboring-segment groups, then
retrieves partial representations to reduce interference from non-overlapping
image content~\citep{segvlad}. We compare its pretrained-feature configuration,
SegVLAD-PreT, in the complete-system table. These approaches motivate different
aggregation units and statistics. TFA focuses on how unlabeled retrieval
evidence can determine assignment, spectral intervention, and global-token
admission within an image-level residual representation.

\paragraph{Cross-view localization.}
Sample4Geo learns cross-view representations with a symmetric contrastive
objective and hard negatives selected using geographic proximity and feature
similarity~\citep{sample4geo2023}. It is the trained reference in our aerial
comparison, using a released University-1652 checkpoint across target
protocols. This contrasts task-trained cross-view transfer with aggregation
of unchanged foundation features; the two systems retain their respective
encoders and training histories.

\paragraph{Statistical calibration.}
Centering and whitening are established retrieval operations
\citep{jegou2012negative}. TFA contributes evidence for admitting and combining
these operations, including database coverage and cross-codebook agreement.
TFA-C64 additionally estimates target-dependent decisions from a disjoint
unlabeled calibration set. In contrast to test-time or source-free adaptation
that optimizes model parameters~\citep{tent,shot}, the encoder remains fixed
and subsequent queries reuse the calibrated rule. We distinguish this
target-data access from database-only adaptation in the experimental protocol.

\section{Reliability-controlled retrieval geometry}
\subsection{Setting and residual representation}
Let $\mathcal D$ denote the reference database and $N_d=|\mathcal D|$ its size.
For each image $x$, a frozen encoder returns $n_x$ patch tokens
$p_i(x)\in\mathbb R^{d_p}$, $i=1,\ldots,n_x$, and a global token $g(x)$;
$d_p$ is the patch-feature dimension. Image arguments are omitted when clear.
TFA fits its controller exclusively on
$\mathcal D$, without retrieval labels, poses, or deployment queries. Its
statistics and decisions are then fixed for independent query inference.
Throughout the main comparisons, TFA denotes this database-only head.
TFA-C64 additionally uses 64 disjoint unlabeled target images. The two
configurations share residual representations, database-fitted PCA, and the
spectral kernel below; their controllers use evidence suited to their available
observations. Appendix~\ref{sec:core-method} defines the target-aware controller.

Following VLAD~\citep{vlad2010}, fit a vocabulary of $K$ centers,
indexed by $k=1,\ldots,K$, and normalize
its centers to $c_k$. Write $\operatorname{L2}(u)=u/\|u\|_2$ for a nonzero
vector $u$, with zero vectors left zero, and $\tilde p_i=\operatorname{L2}(p_i)$.
For assignment type $a\in\{H,S\}$ (Hard or Soft), let $w^a_{ik}$ be the
weight of patch $i$ at center $k$. Hard gives weight one to the center with
largest inner product and zero to the others; Soft uses
$w^S_{ik}=\exp(\langle\tilde p_i,c_k\rangle/\tau)/
\sum_{l=1}^K\exp(\langle\tilde p_i,c_l\rangle/\tau)$,
where $l$ also indexes centers and $\tau>0$ is the temperature.
Soft assignment follows the weighted-quantization principle~\citep{softassignment2008}.
Both use the same residual coordinates with intra-normalization~\citep{intranorm2013}:
\begin{equation}
 v^a=\operatorname{L2}\!\left[\operatorname{L2}\!\left(
 \sum_{i=1}^{n_x} w^a_{ik}(\tilde p_i-c_k)\right)\right]_{k=1}^{K},
 \qquad a\in\{H,S\}.
\end{equation}
Here $[\cdot]_{k=1}^K$ concatenates the $K$ residual blocks, so
$v^a\in\mathbb R^{d_v}$ with $d_v=Kd_p$. The construction index set is
$\mathcal S=\{0,1,2\}$. These three vocabulary fits, fixed before evaluation,
quantify sensitivity to the anchor bank; their indices are suppressed in
descriptor formulas.

\subsection{Database pseudo-query evidence}
A fixed subset $\mathcal P\subseteq\mathcal D$, containing at most 5,000
images, is retrieved against the full database. Each pseudo query excludes
itself and numerically identical residual descriptors. For each assignment
and spectral carrier, we measure the number of distinct retrieved database
images and pairwise top-1 agreement across the three vocabulary constructions.
Coverage measures whether retrieval distinguishes different reference images;
agreement measures whether that behavior persists across vocabulary fits.
Agreement alone also rewards a collapsed retriever that always returns the
same database image. We therefore use coverage and agreement jointly as
unlabeled evidence, rather than interpreting either statistic as correctness.

Hard assignment is admitted over Soft only when it increases support without
reducing agreement in every construction pair. Otherwise Soft is retained.
The same admission test governs global-token fusion. The statistics and
spectral selection rule are defined below; all decisions precede query access.

\subsection{A reversible inner spectral kernel}
For a fixed assignment and construction, database PCA~\citep{pca2016} provides the mean
$\mu\in\mathbb R^{d_v}$, an orthonormal basis $U\in\mathbb R^{d_v\times r}$,
and positive eigenvalues $\lambda_j$, $j=1,\ldots,r$, where $r$ is the retained
rank. Below $v$ denotes a descriptor for this assignment and construction.
For a spectral exponent $\gamma\in[0,1]$, define
\begin{equation}
 \phi_\gamma(v)=\operatorname{L2}
 \left(\operatorname{diag}_{j=1}^{r}(\lambda_j^{-\gamma/2})U^\top(v-\mu)\right).
\end{equation}
The operator $\operatorname{diag}$ constructs a diagonal matrix.
Numerically unsupported directions are excluded. The fixed controls
$\gamma=0,\frac12,1$ denote projected centered features, partial whitening,
and whitening~\citep{jegou2012negative}. At $\gamma=0$, centering and projection remain active;
original-descriptor cosine is defined separately as $S_0$. Scaling by
$\lambda_j^{-\gamma/2}$ progressively reduces the dominance of high-variance
directions. Low-variance directions can also carry noise, so stronger
whitening need not improve discrimination.

For a query image $q$ and database image $d$, define original similarity
$S_0(q,d)=\langle v_q,v_d\rangle$, centered/projected similarity
$S_c(q,d)=\langle\phi_0(v_q),\phi_0(v_d)\rangle$, and weighted similarity
$S_w(c;q,d)=\langle\phi_c(v_q),\phi_c(v_d)\rangle$; $v_q,v_d$ are their
unit-normalized residual descriptors. The scalar $c\in[0,1]$ controls spectral
shaping and $\beta\in[0,1]$ controls centering. Split-database spectral-energy
agreement and construction dispersion determine both
(Appendix~\ref{sec:db-control}). To preserve an exact route back to the
unprojected descriptor, we mix original, centered, and spectrally weighted
similarities. With $\ell(t)=\sin^2(\pi t/2)$ for
$t\in[0,1]$, and suppressing the image arguments, the inner kernel is
\begin{equation}
 S_A=(1-\ell(c))[(1-\ell(\beta))S_0+\ell(\beta)S_c]
       +\ell(c)S_w(c).
 \label{eq:iclr-inner}
\end{equation}
The smooth map $\ell$ takes $[0,1]$ to $[0,1]$ with endpoints zero and one.
It is a design choice; the preservation property below follows from the
nonnegative mixture and its endpoint weights. This construction uses the
standard closure of kernels under nonnegative sums~\citep{aronszajn1950}.
\begin{proposition}[Scope of the fallback]
For fixed deployment strengths and nonzero normalized branch descriptors,
\cref{eq:iclr-inner} is a direct-sum cosine kernel. At $c=\beta=0$, it equals
$S_0$ exactly, regardless of PCA truncation.
\end{proposition}
\begin{proof}
Set $w_0=(1-\ell(c))(1-\ell(\beta))$, $w_c=(1-\ell(c))\ell(\beta)$,
and $w_w=\ell(c)$. All three weights are nonnegative and sum to one.
Define the concatenated feature
\[
\Psi(v)=[\sqrt{w_0}v;\sqrt{w_c}\phi_0(v);\sqrt{w_w}\phi_c(v)].
\]
Its squared norm is $w_0+w_c+w_w=1$, and its inner product is
$w_0S_0+w_cS_c+w_wS_w(c)=S_A$. At $c=\beta=0$ the weights are
$(1,0,0)$, yielding exactly $S_0$ without any requirement on the PCA rank.
\end{proof}
This preservation property applies to the inner kernel. The complete head
subsequently applies outer carrier selection and CLS admission. The proposition
establishes representability and exact fallback, not a guarantee of improved
retrieval accuracy; the selection rules are evaluated experimentally.

\subsection{Support, agreement, and auxiliary evidence}
For a candidate retrieval rule $m$ and construction $s\in\mathcal S$, let
$t_{m,s}(q)$ be the top-1 database index returned for pseudo query
$q\in\mathcal P$. The indicator $\mathbf{1}\{E\}$ equals one when condition
$E$ holds and zero otherwise; $|\cdot|$ denotes set cardinality. Define
\begin{equation}
 U_m=\frac1{|\mathcal S|}\sum_{s\in\mathcal S}|\{t_{m,s}(q):q\in\mathcal P\}|,
 \quad A_m^{ss'}=\frac1{|\mathcal P|}\sum_{q\in\mathcal P}
 \mathbf{1}\{t_{m,s}(q)=t_{m,s'}(q)\}.
\end{equation}
Here $s,s'\in\mathcal S$ are distinct constructions. The first statistic
measures support, the second reproducibility. For candidate rules $m,n$, let
$D_{mn}=\max_{s<s'}(A_m^{ss'}-A_n^{ss'})$, with the maximum over construction
pairs. For spectral selection at a fixed assignment, $m\in\{0,P,A\}$ denotes
original cosine $S_0$, partial whitening $S_P=S_w(1/2)$, or the adaptive
kernel $S_A$, respectively. The deployed carrier is selected in the following
top-to-bottom order ($\land$ means ``and'' and $\lor$ means ``or''). This
decision rule operationalizes the coverage--consistency criterion:
\begin{equation}
 m^*=\begin{cases}
 0,&U_P\le U_0\ \land\ D_{P0}<0,\\
 P,&U_P>U_A\ \lor\ (U_A>U_P\ \land\ D_{AP}<0),\\
 A,&\text{otherwise}.
 \end{cases}
 \label{eq:iclr-selector}
\end{equation}
Thus the original geometry is restored when partial whitening fails both
support and reproducibility. Partial whitening replaces the adaptive carrier
when it expands support, or when adaptive expansion is unanimously less
reproducible. An exact support tie preserves the adaptive carrier.

For two candidate rankings, write $m\succ n$ when $U_m>U_n$ and
$D_{mn}\geq0$. Assignment chooses Hard precisely when its adaptive carrier
satisfies $H\succ S$; here the candidate labels $H,S$ refer to adaptive
retrieval under the two assignments. CLS denotes the encoder's global token
$g$, whose branch score is cosine similarity between normalized global tokens.
CLS is admitted when equal-weight fusion of
database-standardized local and global scores satisfies the same criterion
against the local adaptive score. The scale of each branch is its median
row-wise score standard deviation on database pseudo queries. These scales
and the CLS decision are frozen before outer carrier selection.
Appendix~\ref{sec:db-control} specifies the split statistics and edge conventions.

\section{Experiments}
\subsection{Matched heads and system-level calibration}
\label{sec:eval-protocol}
Training-free heads use DINOv2-B/14~\citep{dinov2} final tokens at $322\times322$ for the
ground comparison, a common database/query protocol, and K64 for VLAD-derived
methods. Vocabulary fitting uses at most 300,000 database patches and spectral
rank is capped at 4,096. Table~\ref{tab:iclr-primary} fixes the backbone;
Table~\ref{tab:iclr-aerial} evaluates both DINOv2-B and DINOv3-B~\citep{dinov3} for every
aerial protocol. Supervised references retain their trained weights.
Appendix~\ref{sec:dataset-sources} lists the original dataset references and
the benchmark releases used for each protocol.
\begin{table}[!t]
\centering\footnotesize
\setlength{\tabcolsep}{2.5pt}
\caption{\textbf{Ground retrieval, R@1 (\%).} Frozen DINOv2-B/14, $322^2$. TFA: DB-only; C64: 64-query calibration (mean$\pm$SD). Protocol: Sec.~\ref{sec:eval-protocol}; dataset sources: Table~\ref{tab:dataset-sources}.}
\label{tab:iclr-primary}
\begin{tabular}{@{}lrrrrrrr@{}}\toprule
 & \multicolumn{2}{c}{Supervised References} & \multicolumn{5}{c}{Training Free Systems}\\
\cmidrule(lr){2-3}\cmidrule(lr){4-8}
Dataset & SALAD & SelaVPR++ & AnyLoc & TF-VPR & RIA$^d$ & TFA & TFA-C64\\\midrule
\rowcolor{gray!15}\multicolumn{8}{l}{\textbf{Urban / ground}}\\
Pitts30k & 92.37$^r$ & 93.21$^r$ & 82.12 & \underline{84.93} & 81.84 & \textbf{86.01} & 84.33$\pm$1.59\\
Pitts250k & 95.08$^r$ & 95.95$^r$ & 82.97 & 82.96 & \underline{84.73} & \textbf{87.01} & 83.70$\pm$1.63\\
MSLS-val & 92.03$^r$ & 93.78$^r$ & 52.70 & 41.49 & 44.41 & \underline{70.09} & \textbf{71.16}$\pm$2.55\\
St-Lucia & 100.00$^r$ & 99.93$^r$ & 92.31 & 88.11 & 85.31 & \textbf{97.61} & \underline{95.72}$\pm$2.87\\
Eynsham & 91.59$^r$ & 92.24$^r$ & 76.45 & 71.15 & 73.13 & \textbf{88.45} & \underline{85.19}$\pm$3.92\\
GSV-Brussels & 95.53$^r$ & 94.48$^r$ & 74.04 & 63.06 & 63.66 & \underline{84.36} & \textbf{84.75}$\pm$0.90\\
Essex3in1 & 90.00$^r$ & 91.43$^r$ & 78.25 & 73.81 & 71.11 & \textbf{82.38} & \underline{82.21}$\pm$0.85\\
V4RL shopping & 94.00$^r$ & 93.53$^r$ & \textbf{89.05} & 82.97 & \underline{86.97} & 69.54 & 86.91$\pm$2.66\\
\rowcolor{gray!15}\multicolumn{8}{l}{\textbf{Condition shift}}\\
SPED & 92.09$^r$ & 90.77$^r$ & 75.57 & 73.48 & 69.85 & \textbf{85.12} & \underline{83.79}$\pm$0.63\\
Nordland & 86.46$^r$ & 94.89$^r$ & 29.50 & 29.46 & 21.03 & \textbf{39.60} & \underline{30.23}$\pm$3.27\\
AmsterTime & 58.57$^r$ & 57.27$^r$ & \underline{45.71} & 41.35 & 36.50 & 44.84 & \textbf{52.00}$\pm$1.76\\
CrossSeason & 100.00$^r$ & 100.00$^r$ & 99.83 & \textbf{100.00} & 99.65 & \textbf{100.00} & \underline{99.97}$\pm$0.08\\
\rowcolor{gray!15}\multicolumn{8}{l}{\textbf{Driving}}\\
4Seasons Business Campus & 98.35$^r$ & 98.65$^r$ & \underline{94.84} & 88.27 & 84.06 & 51.23 & \textbf{95.55}$\pm$0.65\\
4Seasons City Loop & 95.09$^r$ & 98.32$^r$ & \underline{78.25} & 62.79 & 74.03 & 60.59 & \textbf{82.88}$\pm$0.77\\
4Seasons Countryside & 73.53$^r$ & 85.06$^r$ & 28.05 & 32.37 & \underline{35.12} & 12.58 & \textbf{36.76}$\pm$5.75\\
4Seasons Office Loop & 99.16$^r$ & 98.88$^r$ & \textbf{90.50} & 78.51 & 79.49 & 7.91 & \underline{90.20}$\pm$0.75\\
4Seasons Old Town & 88.96$^r$ & 86.94$^r$ & \textbf{70.24} & 64.54 & 65.58 & 9.75 & \underline{70.20}$\pm$1.87\\
4Seasons Parking Garage & 100.00$^r$ & 100.00$^r$ & \textbf{99.89} & 99.01 & 99.23 & 97.21 & \underline{99.30}$\pm$0.28\\
RobotCar nine conditions & 92.84$^r$ & 92.39$^r$ & \underline{87.82} & 78.93 & 86.63 & 69.46 & \textbf{88.33}$\pm$0.46\\
\rowcolor{gray!15}\multicolumn{8}{l}{\textbf{Indoor}}\\
Baidu-Mall & 68.41$^r$ & 67.98$^r$ & \underline{63.26} & 51.22 & 58.04 & 56.14 & \textbf{63.55}$\pm$1.71\\
\bottomrule\end{tabular}
\par\smallskip\raggedright\scriptsize $d$: RIA reconstruction; $r$: checkpoint inference. Bold/underline: best/second-best training-free result (C64 mean). C64 excludes calibration images; other columns use full query sets. RobotCar uses 64 calibration images across nine conditions.
\end{table}
\begin{table}[!t]
\centering\footnotesize
\caption{\textbf{Aerial / cross-view retrieval, R@1/5/10 (\%).} TFA: DB-only. Bold/underline rank training-free heads within each backbone. Calibration variants: Table~\ref{tab:iclr-aerial-access}; dataset sources: Table~\ref{tab:dataset-sources}.}
\label{tab:iclr-aerial}
\setlength{\tabcolsep}{0pt}\fontsize{8.4}{10}\selectfont
\begin{tabular}{@{}l@{\hspace{6pt}}r@{\hspace{2pt}}r@{\hspace{2pt}}r*{4}{@{\hspace{7pt}}r@{\hspace{2pt}}r@{\hspace{2pt}}r}@{}}\toprule
 & \multicolumn{3}{c}{Trained Reference} & \multicolumn{12}{c}{Training Free Systems}\\
\cmidrule(lr){2-4}\cmidrule(lr){5-16}
 & \multicolumn{3}{c}{Sample4Geo} & \multicolumn{3}{c}{AnyLoc} & \multicolumn{3}{c}{TF-VPR} & \multicolumn{3}{c}{RIA$^d$} & \multicolumn{3}{c}{TFA}\\
\cmidrule(lr){2-4}\cmidrule(lr){5-7}\cmidrule(lr){8-10}\cmidrule(lr){11-13}\cmidrule(lr){14-16}
Backbone & R@1 & R@5 & R@10 & R@1 & R@5 & R@10 & R@1 & R@5 & R@10 & R@1 & R@5 & R@10 & R@1 & R@5 & R@10\\\midrule
\rowcolor{gray!15}\multicolumn{16}{l}{\textbf{UAV-VisLoc}}\\
DINOv2-B & \multirow{2}{*}{10.44$^r$} & \multirow{2}{*}{21.33} & \multirow{2}{*}{27.16} & 12.98 & \underline{27.61} & \underline{34.89} & 6.02 & 14.96 & 20.20 & \underline{13.16} & 26.27 & 32.64 & \textbf{15.63} & \textbf{34.02} & \textbf{41.59}\\
DINOv3-B &  &  &  & \underline{14.38} & \underline{30.31} & \underline{38.25} & 10.52 & 23.34 & 29.54 & 12.76 & 23.49 & 29.29 & \textbf{19.94} & \textbf{37.80} & \textbf{45.17}\\
\addlinespace[2pt]
\rowcolor{gray!15}\multicolumn{16}{l}{\textbf{VPAIR-full}}\\
DINOv2-B & \multirow{2}{*}{37.66$^r$} & \multirow{2}{*}{49.59} & \multirow{2}{*}{56.25} & \underline{32.20} & \underline{50.39} & \underline{58.62} & 23.73 & 41.61 & 49.93 & 29.76 & 44.53 & 51.43 & \textbf{41.20} & \textbf{61.64} & \textbf{69.96}\\
DINOv3-B &  &  &  & 33.96 & 50.48 & 58.36 & \underline{36.47} & \underline{55.36} & \underline{64.19} & 26.55 & 39.27 & 46.54 & \textbf{42.10} & \textbf{62.47} & \textbf{71.24}\\
\addlinespace[2pt]
\rowcolor{gray!15}\multicolumn{16}{l}{\textbf{SUES-200}}\\
DINOv2-B & \multirow{2}{*}{80.29$^r$} & \multirow{2}{*}{94.38} & \multirow{2}{*}{97.14} & 48.75 & 76.18 & 85.54 & 51.22 & 75.89 & 84.77 & \underline{54.23} & \underline{78.97} & \underline{87.16} & \textbf{58.95} & \textbf{81.05} & \textbf{87.43}\\
DINOv3-B &  &  &  & 59.28 & 82.40 & 89.25 & \underline{63.24} & \underline{85.46} & \underline{91.09} & 49.95 & 72.73 & 81.65 & \textbf{73.67} & \textbf{91.43} & \textbf{95.66}\\
\addlinespace[2pt]
\rowcolor{gray!15}\multicolumn{16}{l}{\textbf{DenseUAV}}\\
DINOv2-B & \multirow{2}{*}{23.68$^r$} & \multirow{2}{*}{47.53} & \multirow{2}{*}{57.53} & 5.48 & 17.90 & 26.74 & \underline{7.16} & \underline{24.62} & \underline{36.72} & 6.05 & 19.68 & 29.73 & \textbf{8.84} & \textbf{26.93} & \textbf{39.47}\\
DINOv3-B &  &  &  & 4.26 & 16.06 & 24.82 & \underline{7.21} & \underline{23.98} & \underline{35.82} & 5.95 & 19.23 & 28.47 & \textbf{8.19} & \textbf{27.17} & \textbf{39.45}\\
\addlinespace[2pt]
\rowcolor{gray!15}\multicolumn{16}{l}{\textbf{Park}}\\
DINOv2-B & \multirow{2}{*}{26.52$^r$} & \multirow{2}{*}{62.48} & \multirow{2}{*}{75.62} & \underline{34.49} & \underline{70.29} & \underline{82.50} & 33.19 & 69.66 & 79.98 & 29.56 & 66.28 & 78.83 & \textbf{38.98} & \textbf{77.72} & \textbf{88.42}\\
DINOv3-B &  &  &  & \textbf{46.12} & \underline{76.42} & \underline{85.23} & 38.61 & 71.67 & 81.49 & 33.74 & 70.16 & 80.27 & \underline{46.02} & \textbf{78.22} & \textbf{86.59}\\
\addlinespace[2pt]
\rowcolor{gray!15}\multicolumn{16}{l}{\textbf{Urbanscape}}\\
DINOv2-B & \multirow{2}{*}{67.83$^r$} & \multirow{2}{*}{87.43} & \multirow{2}{*}{90.98} & \underline{76.95} & \underline{92.56} & \underline{95.38} & 61.24 & 82.62 & 88.82 & 69.71 & 86.13 & 91.03 & \textbf{84.93} & \textbf{95.68} & \textbf{97.25}\\
DINOv3-B &  &  &  & \underline{82.41} & \underline{93.48} & \underline{95.43} & 69.70 & 87.87 & 91.83 & 74.51 & 88.29 & 92.27 & \textbf{87.27} & \textbf{95.12} & \textbf{96.43}\\
\addlinespace[2pt]
\rowcolor{gray!15}\multicolumn{16}{l}{\textbf{University drone$\to$sat.}}\\
DINOv2-B & \multirow{2}{*}{92.00$^r$} & \multirow{2}{*}{97.55} & \multirow{2}{*}{98.13} & 27.59 & 43.01 & 50.01 & 22.08 & 36.56 & 43.70 & \underline{31.06} & \underline{49.01} & \underline{57.22} & \textbf{44.29} & \textbf{61.81} & \textbf{68.78}\\
DINOv3-B &  &  &  & \underline{42.49} & \underline{65.64} & \underline{73.94} & 41.16 & 63.46 & 71.49 & 39.50 & 58.17 & 65.66 & \textbf{46.81} & \textbf{67.10} & \textbf{74.03}\\
\addlinespace[2pt]
\rowcolor{gray!15}\multicolumn{16}{l}{\textbf{University sat.$\to$drone}}\\
DINOv2-B & \multirow{2}{*}{95.15$^l$} & \multirow{2}{*}{97.00} & \multirow{2}{*}{97.29} & 45.89 & 58.54 & 62.53 & 36.09 & 51.93 & 58.63 & \textbf{66.76} & \textbf{78.08} & \textbf{81.98} & \underline{53.54} & \underline{65.05} & \underline{70.66}\\
DINOv3-B &  &  &  & 71.14 & 81.46 & 85.92 & 62.62 & 77.18 & 81.03 & \underline{82.31} & \underline{90.16} & \underline{92.06} & \textbf{83.93} & \textbf{92.34} & \textbf{94.63}\\
\addlinespace[2pt]
\bottomrule\end{tabular}
\par\smallskip\raggedright\scriptsize UAV-VisLoc: coverage-qualified 100m recall; Park/Urbanscape: 50m XY; others: registered identities/positives. SUES averages four heights. Park AnyLoc uses seed0; Urbanscape AnyLoc uses seed0 for D2 and three seeds for D3. TFA/RIA use three replicas.
\par\scriptsize Sample4Geo: University-trained ConvNeXt-B, $384^2$, shared across backbone rows; $r$: checkpoint inference; $l$: author log. Native interfaces: Appendix~\ref{sec:native-interfaces}.
\end{table}

TFA uses database-only calibration; C64 fits the query-calibrated controller
on 64 unlabeled images and evaluates the disjoint remainder.
The variants share residual representations and the spectral kernel, while
their assignment and CLS decision rules differ. Their comparison assesses
deployed heads; it does not isolate query access alone.
Throughout, mean$\pm$SD denotes the mean and sample standard deviation across
ten random splits, averaging three codebook recalls within each split.
Bold/underline rank reported means; calibration access and evaluation subsets
may differ. TFA-FullQ uses all unlabeled queries. Appendix
Tables~\ref{tab:full-query-reference} and~\ref{tab:iclr-aerial-access}
compare these variants; controller definitions are in Appendices
\ref{sec:db-control}--\ref{sec:core-method}.

The database-only head exceeds AnyLoc, TF-VPR, and RIA on nine ground
protocols and ties the best of these baselines on CrossSeason. Its driving results expose
a limitation of database pseudo-query evidence: on Office Loop and Old Town,
R@1 falls to 7.91 and 9.75, respectively. The query-calibrated variants avoid
much of this degradation. Database consistency therefore provides useful
adaptation evidence but does not ensure transfer to incoming queries.
See Appendix~\ref{sec:office-loop-failure} for the Office Loop failure analysis.

\paragraph{Which head benefits from a new representation?}
Switching to DINOv3 improves SUES R@1 by 14.72 points for TFA and 10.53 for
AnyLoc. On University satellite-to-drone, TFA gains 30.39 points versus
15.55 for RIA: RIA leads with DINOv2-B, TFA with DINOv3-B. Thus representation
gains depend on the aggregation head and scene. Supplementary ground
comparisons use TFA-FullQ.
\paragraph{Trained reference systems.}
SALAD and SelaVPR++ provide ground references; Sample4Geo provides aerial
context using one University-trained checkpoint without target fine-tuning.
Sample4Geo is stronger on SUES-200, DenseUAV, and University; TFA is stronger
on UAV-VisLoc, VPAIR, Park, and Urbanscape. These native-system comparisons
show scene-dependent transfer under distinct training and feature interfaces.

Figure~\ref{fig:iclr-qualitative} compares database-calibrated TFA and baseline retrievals
across five scene families.

\begin{figure}[t]
 \centering\includegraphics[width=\linewidth]{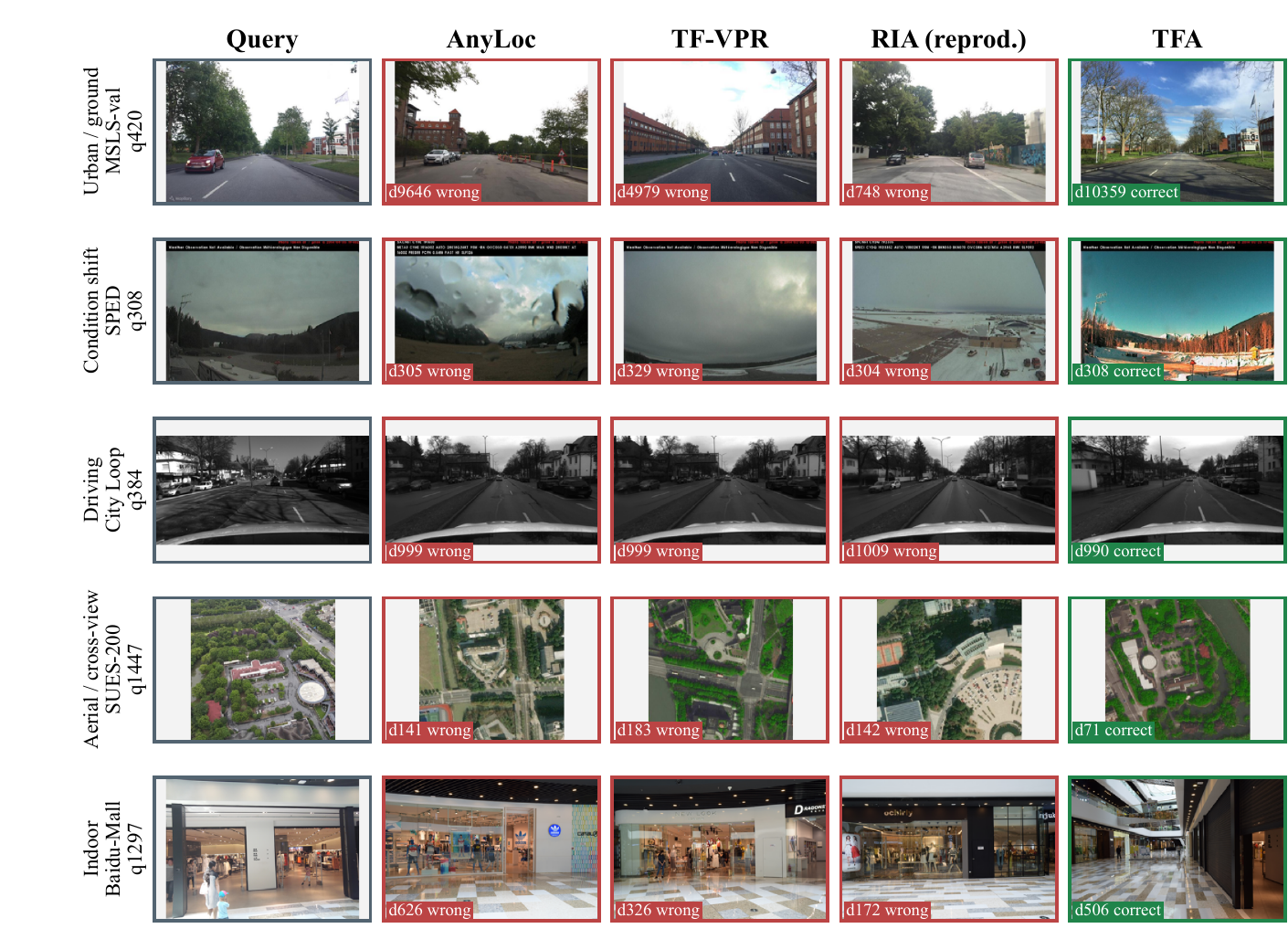}
 \caption{Selected Top-1 retrievals across five scene families.
 DINOv2-B, seed 0, DB-only TFA; green/red denote correct/incorrect matches.
 Each selected query is recovered by TFA but missed by all three baselines.
 RIA uses our reconstruction. Aggregate results: Tables~\ref{tab:iclr-primary}
 and~\ref{tab:iclr-aerial}.}
 \label{fig:iclr-qualitative}
\end{figure}

Additional successes and complementary failures across the same five scene
families appear in Appendix Figures~\ref{fig:supp-retrieval-success}
and~\ref{fig:supp-retrieval-failure}.

\subsection{Small disjoint calibration sets}
The TFA-C64 system combines DINOv2-G L31 value features at $640\times480$,
source/map vocabularies, and target-score calibration. Statistics are fitted
on the disjoint calibration partition and fixed for subsequent queries.
\begin{table}[t]
\centering\footnotesize
\setlength{\tabcolsep}{2.6pt}
\caption{\textbf{Native-system comparison, R@1 (\%).} TFA uses the dual-vocabulary DINOv2-G L31 value configuration at $640\times480$. TFA-C$N$: $N$ disjoint calibration images (mean$\pm$SD); TFA-FullQ (Cross-fit): two-fold query cross-fitting.}
\label{tab:iclr-strongest}
\begin{tabular}{@{}llrrrrr@{}}\toprule
System & Backbone/interface & Pitts30k & SPED & St-Lucia & SVOX-5 & VPAIR-full\\\midrule
\rowcolor{gray!15}\multicolumn{7}{l}{\textbf{Supervised References}}\\
BoQ$^p$ & native &93.7&92.5&100&98.38&29.3\\
SALAD$^p$ & DINOv2-B &92.4&92.1&100&97.66&22.1\\
SelaVPR++ & DINOv2-L (perf.) &94.4$^p$&92.75$^r$&100.00$^r$&98.54$^r$&42.39$^r$\\
\rowcolor{gray!15}\multicolumn{7}{l}{\textbf{Training Free Systems}}\\
TF-VPR$^p$ & DINOv2-B &84.3&77.6&90.8&58.54&67.5\\
AnyLoc & DINOv2-G L31 &\underline{87.7}$^p$&85.50$^r$&96.2$^p$&85.84$^r$&66.7$^p$\\
SegVLAD-PreT & G + SAM-H &86.70$^p$&\underline{88.63}$^r$&96.79$^r$&\underline{86.09}$^r$&\underline{69.80}$^p$\\
RIA & DINOv2-G L31 &86.70$^p$&68.37$^r$&\underline{97.20}$^p$&41.71$^r$&45.23$^r$\\
TFA-C64 & DINOv2-G L31 &\textbf{91.46}$\pm$.27&\textbf{91.01}$\pm$.69&\textbf{99.26}$\pm$.15&\textbf{91.88}$\pm$.34&\textbf{76.29}$\pm$.39\\
\midrule
\rowcolor{gray!10}\multicolumn{7}{l}{Calibration budget for the dual-vocabulary system}\\
\quad TFA-C8 & &87.91$\pm$.86&87.48$\pm$.48&97.12$\pm$.57&85.65$\pm$.48&62.29$\pm$1.66\\
\quad TFA-C16 & &90.24$\pm$.49&89.64$\pm$.60&98.51$\pm$.29&90.37$\pm$.45&71.90$\pm$1.07\\
\quad TFA-C32 & &90.85$\pm$.37&90.57$\pm$.69&99.05$\pm$.21&91.54$\pm$.36&74.68$\pm$.54\\
\quad TFA-C128 & &91.74$\pm$.19&91.57$\pm$.71&99.33$\pm$.14&92.09$\pm$.29&76.99$\pm$.50\\
\quad TFA-C256 & &91.88$\pm$.24&91.93$\pm$1.16&99.42$\pm$.14&92.06$\pm$.26&77.58$\pm$.48\\
\quad TFA-FullQ (Cross-fit) & &92.09&93.08&99.50&92.15&78.31\\
\bottomrule
\end{tabular}
\par\smallskip\raggedright\scriptsize $p$: published; $r$: our inference of comparison methods. Bold/underline: best/second-best training-free system at TFA-C64; additional budgets are unranked. SVOX: five-condition macro. TFA-C$N$ excludes calibration images; other methods use native query sets. Interfaces and RIA reconstruction details: Appendix~\ref{sec:native-interfaces}.
\end{table}

Table~\ref{tab:iclr-strongest} compares native systems and calibration budgets.
TFA-C64 is competitive on conventional benchmarks and strong on VPAIR.
Calibration-size curves appear in Appendix Figure~\ref{fig:iclr-calibration}.
The single-vocabulary B head is separately tested with identical inputs and
paired evaluation subsets in Appendix~\ref{sec:paired-single-calibration}.

Across the five map units, 64 calibration images preserve 97.41--99.76\% of
paired full-query absolute R@1. Complementarily, net-gain recovery measures
the retained improvement over the uncalibrated system; it is 50.3\% on SPED.
SVOX uses condition-specific calibration.
\begin{table}[t]
\centering\footnotesize
\setlength{\tabcolsep}{3pt}
\renewcommand{\arraystretch}{1.08}
\caption{Spectral controls and complete heads, R@1 (\%). Controls fix the TFA-FullQ assignment; adaptive spectrum is $S_A$ before selection and CLS. C64: mean$\pm$SD. Bold/underline: best/second-best.}
\label{tab:iclr-spectral}
\label{tab:iclr-mapfold}
\begin{tabular}{@{}lrrrrrrr@{}}\toprule
 & \multicolumn{4}{c}{Spectral controls (TFA-FullQ assignment)}
 & \multicolumn{3}{c}{Complete heads}\\
\cmidrule(lr){2-5}\cmidrule(lr){6-8}
Map & Original & Partial & Full & Adaptive & TFA & TFA-C64 & TFA-FullQ\\
 & cosine & whitening & whitening & spectrum & & &\\\midrule
\rowcolor{gray!12}\multicolumn{8}{l}{(a) Fixed transformations have map-dependent effects}\\
Pitts30k & 83.97 & \underline{86.00} & 85.00 & 85.64 & \textbf{86.01} & 84.33$\pm$1.59 & \underline{86.00}\\
SPED & 76.00 & 83.31 & \textbf{85.34} & 85.01 & \underline{85.12} & 83.79$\pm$0.63 & 85.01\\
AmsterTime & 45.71 & \underline{52.69} & 0.76 & 48.87 & 44.84 & 52.00$\pm$1.76 & \textbf{52.72}\\\midrule
\rowcolor{gray!12}\multicolumn{8}{l}{(b) 4Seasons: maps held out from the TFA-FullQ outer-selector design}\\
City Loop & 80.66 & \textbf{83.94} & 56.16 & 82.60 & 60.59 & \underline{82.88}$\pm$0.77 & \textbf{83.94}\\
Office Loop & \underline{90.50} & 89.09 & 7.87 & \textbf{90.54} & 7.91 & 90.20$\pm$0.75 & \textbf{90.54}\\
Old Town & \textbf{70.71} & 67.39 & 7.62 & \textbf{70.71} & 9.75 & \underline{70.20}$\pm$1.87 & \textbf{70.71}\\
Parking Garage & \textbf{99.51} & 98.36 & 96.99 & 97.59 & 97.21 & \underline{99.30}$\pm$0.28 & \textbf{99.51}\\\midrule
Four-map mean & 85.34 & 84.69 & 42.16 & 85.36 & 43.87 & \underline{85.64}$\pm$0.51 & \textbf{86.17}\\\bottomrule
\end{tabular}
\par\smallskip\raggedright\scriptsize C64 excludes calibration images; other columns use full query sets.
Spectral controls fix the assignment; complete heads compare deployed configurations with different controllers and calibration access. TFA-FullQ rejects CLS in panel (b); its maps were held out from outer-selector design. C64 macro statistics average maps within each split.

\end{table}

\paragraph{Calibration transfers to held-out queries.}
With the feature interface fixed, database calibration gives
84.37/96.04/50.71 R@1 on Pitts30k/St-Lucia/VPAIR, versus 92.09/99.50/78.31
with even/odd query cross-fitting. Each query is excluded from its own
statistics; cross-fitting remains within 0.29/0.18/0.17 points of full-query
calibration. This isolates the benefit of target statistics in that system.

\subsection{Why the operator decisions matter}
Table~\ref{tab:iclr-spectral} shows why a fixed spectral transformation is
insufficient: whitening improves SPED but collapses AmsterTime. Under
query-calibrated control, support--agreement selection retains the adaptive
spectrum on SPED and partial whitening on Pitts30k and AmsterTime.
On the four held-out 4Seasons maps, the TFA-FullQ outer selector matches the
best displayed spectral candidate, whereas database-only calibration can
fail to transfer (Appendix~\ref{sec:office-loop-failure}).
Assignment and CLS provide complementary controls: single-anchor assignment
helps the Nardo-D3 stress case but hurts a City Loop split; unconditional
CLS fusion degrades eight of 17 development cells, while the query-aware
retention veto improves the mean by 2.17 points over the local branch.
Appendix~\ref{sec:component-analysis} gives the component-specific protocols
and evidence ablations.

\subsection{Deployment cost}
On RTX 3090, the matched AmsterTime benchmark measures 12.468 ms for the
DINOv2-B encoder and 1.223 ms for the TFA-FullQ head, including Top-10 search;
the head adds 0.410 ms over single-anchor VLAD. Separate H100 measurements
quantify database-only and C64 setup and retrieval costs.
Appendix~\ref{sec:deployment-timing} specifies both timing protocols.

\section{Discussion and conclusion}
TFA adapts frozen-feature aggregation using only an unlabeled reference
database. Its support--agreement controller and reversible spectral kernel
provide a practical baseline for new retrieval environments without task-specific
weight training. The aerial comparisons show strong performance across two
foundation models; the ground results reveal that database consistency can
misrepresent deployment conditions. Query-calibrated variants quantify the
benefit of target observations under an explicitly different access protocol.
The distinction between reference reliability and transfer reliability is
central to further progress in training-free retrieval.

\section*{AI use statement}
Generative AI tools assisted with methodological discussions, experimental
design, mathematical exposition, implementation and debugging, interpretation
of results, literature discovery, figure preparation, and manuscript editing.
Reported retrieval metrics are computed by evaluation programs from model
predictions and dataset annotations. The authors are responsible for the
accuracy, originality, and reproducibility of the final submission.

\bibliography{refs,scene_refs,source_refs}
\bibliographystyle{iclr2027_conference}
\clearpage\appendix
\section{Database-only controller}
\label{sec:db-control}
TFA calibrates its aggregation head on the reference database and then
processes incoming queries independently.
Each assignment uses three database-fitted K64 vocabularies, temperature
$\tau=0.01$ for Soft, unit-normalized centers, and the residual definition
in the main text. PCA retains at most 4,096 numerically supported directions.
The fixed sample used to fit Pitts250k PCA contains 5,000 database images;
retrieval still uses the full gallery.

\paragraph{Split-database spectral statistics.}
A common database-fitted PCA basis defines the coordinates for measuring
construction stability. A seed-0 permutation partitions the database into alternating halves
$\mathcal D_0,\mathcal D_1$, indexed by $h\in\{0,1\}$.
For construction $s\in\mathcal S$, $U_s,\mu_s$ are its PCA basis and mean,
$v_d$ is the normalized residual descriptor of database image $d$, and $j$
indexes a retained PCA coordinate. These statistics are computed separately
for each assignment.
Let $e_{s,h}$ be the vector with entries
$e_{s,h,j}=|\mathcal D_h|^{-1}\sum_{d\in\mathcal D_h}
[U_s^\top(v_d-\mu_s)]_j^2$.
Let $\operatorname{corr}$ denote Pearson correlation and
$\operatorname{clip}(x,a,b)=\min(b,\max(a,x))$. Define
$\rho_s=\operatorname{corr}(e_{s,0},e_{s,1})$ and
$\alpha_s=\operatorname{clip}(\rho_s,0,1)^{2\mathrm{PR}^u_s}$,
where $\mathrm{PR}^u_s=(\sum_j\nu_{s,j})^2/\sum_j\nu_{s,j}^2$ and
$\nu_{s,j}$ are the eigenvalues of the uncentered database Gram matrix
formed from the normalized descriptors. This participation ratio measures
spectral dimensionality~\citep{participation2017}.
The correlation measures agreement of per-direction energies. Raising its
clipped value to $2\mathrm{PR}^u_s$ makes the score more conservative when
energy is spread over many directions: at any correlation strictly between
zero and one, increasing the exponent lowers the score. The exponent is a
fixed reliability heuristic, not a probability derived from a noise model.
Undefined correlations are set to zero.
For a vector $x=(x_0,x_1,x_2)$ of construction-level values, define
$\operatorname{MAD}(x)=\operatorname{median}_s|x_s-\operatorname{median}(x)|$.
The normal-consistent MAD scale follows~\citet{mad1993}.
Let $\Phi$ be the standard normal cumulative distribution function and
$\Phi^{-1}$ its quantile function. The dispersion-penalized statistic is
\begin{equation}
 L(x)=\operatorname{clip}\left[
 \operatorname{median}(x)-
 \frac{\Phi^{-1}(.975)\operatorname{MAD}(x)}
 {\Phi^{-1}(.75)\sqrt{3}},0,1\right].
\end{equation}
Here $L$ is a dispersion-penalized reliability score with a normal-reference
scale. With only three dependent vocabulary fits, it is not a calibrated
95\% confidence bound. Its role is to reduce intervention when constructions
disagree. Let $\mu_{s,h}$ be the mean normalized
descriptor in half $h$, and let $\alpha=(\alpha_s)_{s\in\mathcal S}$ and
$\rho=(\rho_s)_{s\in\mathcal S}$. Then
\begin{equation}
 c=L(\alpha),\qquad
 \beta=\operatorname{clip}\left[
 \operatorname{median}_s
 \frac{\|\mu_{s,0}-\mu_{s,1}\|_2}{\max(\|\mu_s\|_2,10^{-30})}
 (1-L(\rho)),0,1\right].
\end{equation}
These strengths enter Eq.~\ref{eq:iclr-inner}.
Thus $c$ increases with reproducible spectral energy, while $\beta$ requires
both a mean displacement and reduced energy agreement. These definitions
specify how evidence sets intervention strength; the kernel's algebraic
fallback property holds independently of this particular scoring rule.

\paragraph{Pseudo-query decisions.}
All database images are used as pseudo queries when $N_d\leq5,000$;
otherwise a fixed seed-0 subset of 5,000 is used. The gallery is unchanged.
We exclude self-matches and numerical duplicates, identified by cosine
similarity within $5\times10^{-13}$ of one in float64 arithmetic.
Hard is selected only if its adaptive score increases distinct-image support
and has a nonnegative agreement change in at least one codebook pair;
otherwise Soft is selected. Original cosine, partial whitening, and the
adaptive kernel then enter Eq.~\ref{eq:iclr-selector}.

For CLS admission, the local adaptive score and CLS cosine are separately
divided by their median pseudo-query row standard deviations, computed
over unmasked gallery entries with population normalization. Their equal
mixture is admitted by the same support--agreement test against the local
adaptive score. These two scales and the admission decision are fixed before
replacing the local carrier. Denote these fixed scales by
$\sigma_{\rm local}$ and $\sigma_{\rm CLS}$, the selected local score by
$S_{m^*}$, and global-token cosine by $S_{\rm CLS}$.
At inference an admitted mixture is
$\tfrac12(S_{m^*}/\sigma_{\rm local}+S_{\rm CLS}/\sigma_{\rm CLS})$;
otherwise the selected local score is used. Exact ranking ties use ascending
database index. Reported recalls average the three constructions.

\paragraph{Target-domain calibration.}
When 64 unlabeled target images are available, TFA-C64 measures transfer
from the database to this disjoint calibration set. It shares the residual
representation and spectral kernel with TFA, but uses capacity--relation
assignment and quality-weighted CLS fusion with a retention test
(Appendix~\ref{sec:core-method}). Thus the two deployed variants differ in
controller as well as calibration access. All decisions are fixed before
evaluating subsequent queries.

\paragraph{Failure case: within-database reliability does not ensure transfer.}
\label{sec:office-loop-failure}
On 4Seasons Office Loop, database-only TFA selects Soft assignment, the
adaptive spectrum with $c=0.9814$, and CLS fusion, obtaining 7.91\% R@1.
Replacing only the spectral score by original cosine, while retaining the
assignment and CLS decision and scales, recovers 90.54\%; fixed partial
whitening reaches 89.28\%. This intervention identifies spectral shaping as
the principal source of the degradation despite favorable database support
and cross-codebook agreement. In contrast, TFA-C64 estimates
$c\in[0.084,0.296]$ for its selected assignment across ten splits and obtains
90.20$\pm$0.75\% on the disjoint evaluation complements. Six splits retain
the adaptive kernel, two select original cosine, and two select partial
whitening. In one adaptive split, the spectral mixture weight falls from
99.91\% for database-only TFA to 1.73\%, preserving predominantly original
similarity. C64 also changes assignment and rejects CLS in all ten splits;
their individual contributions are not isolated by this comparison.
The failure demonstrates that reproducible database discrimination can
favor a spectral geometry that does not transfer to target observations;
target-domain calibration can substantially reduce the resulting intervention.

\section{TFA-C64: calibration for independent query inference}
\label{sec:core-method}
TFA-C64 uses 64 unlabeled target-domain images to calibrate a fixed head
for subsequent retrieval. Vocabulary fitting and PCA remain database-only;
target images determine assignment, spectral strengths, and CLS admission.
Calibration images are excluded from evaluation, and evaluation queries
do not update the head.

We use the dimensions, normalization, assignment labels, and construction set
defined in the main text; bold vectors denote the same quantities.
Let $\mathcal Q$ be the disjoint unlabeled calibration set, $N_q=|\mathcal Q|=64$, and
$N_d=|\mathcal D|$. All query statistics below refer to calibration images
in $\mathcal Q$; database-index sums run
over $j=1,\ldots,N_d$, unless PCA coordinates are specified.
The positive scalar $\epsilon$ denotes a numerical stabilizer in the stated
statistic. Clipping, MAD, and the standard-normal quantile $\Phi^{-1}$ follow
Appendix~\ref{sec:db-control}; $\mathbf{1}\{E\}$ is the indicator of event $E$.

\subsection{Residual hypotheses}
For each map, row-normalized database patch tokens are used to fit a visual
vocabulary $\mathcal C=\{\mathbf c_k\}_{k=1}^{K}$.  A predefined set of
independent codebook constructions $s\in\mathcal S$ quantifies vocabulary
uncertainty.

Database-only K-means~\citep{lloyd1982} gives fitted centers $\boldsymbol\mu_k$. We use
spherical residual coordinates $\mathbf c_k=\operatorname{L2}(\boldsymbol\mu_k)$
for both assignment and subtraction. For a normalized patch
$\tilde{\mathbf p}_i$, we then form single-anchor (Hard) and multi-anchor
(Soft) assignments
\begin{align}
 w^{H}_{ik} &= \mathbf{1}\!\left\{k=\arg\max_{1\leq l\leq K}
   \langle\tilde{\mathbf p}_i,\mathbf c_l\rangle\right\},\\
 w^{S}_{ik} &=
 \frac{\exp(\langle\tilde{\mathbf p}_i,\mathbf c_k\rangle/\tau)}
 {\sum_l\exp(\langle\tilde{\mathbf p}_i,\mathbf c_l\rangle/\tau)},
 \qquad \tau>0 .
 \label{eq:assignment-definitions}
\end{align}
The corresponding VLAD residual blocks~\citep{vlad2010} and
intra-normalized descriptor~\citep{intranorm2013} are
\begin{equation}
 \mathbf v_k^a=\sum_i w^a_{ik}(\tilde{\mathbf p}_i-\mathbf c_k),\qquad
 \mathbf v^a=\operatorname{L2}\!\left(
 [\operatorname{L2}(\mathbf v_1^a);\ldots;
  \operatorname{L2}(\mathbf v_K^a)]\right),
 \label{eq:vlad}
\end{equation}
where $a\in\{H,S\}$.  Hard assignment removes all non-winning anchor mass;
Soft assignment preserves graded residual evidence.  Their selection therefore
cannot be based only on which construction appears more stable.

\paragraph{Shared geometry and controller-specific evidence.}
The residual coordinates and spectral mixture are shared with the database-only
head. The formulas below define how the target-aware controller uses a
calibration set to assess transfer. Information ratios and joint evidence
scores are operational design choices; their role is evaluated through
assignment and CLS ablations, separately from the spectral-kernel proof.

\subsection{Capacity-coherent assignment}
\label{sec:capacity-assignment}

\paragraph{Hub-corrected information retention.}
Let $t_{a,s}(q)$ be the top-1 database index of calibration image $q$ under assignment $a$
and codebook $s$.  Its empirical distribution over the three constructions is
$p_a(j\mid q)$, and $p_a(j)=N_q^{-1}\sum_qp_a(j\mid q)$ is the pooled database
occupancy; $p_a(j\mid q)=|\mathcal S|^{-1}\sum_{s\in\mathcal S}
\mathbf{1}\{t_{a,s}(q)=j\}$. Using a stabilized, normalized
Kullback--Leibler-type score~\citep{kullback1951}, we define
\begin{equation}
 I_a(q)=\frac{1}{\log\min(N_q,N_d)}
 \sum_jp_a(j\mid q)\log\frac{p_a(j\mid q)}{p_a(j)+\epsilon} .
 \label{eq:retrieval-information}
\end{equation}
The marginal correction prevents a repeatedly retrieved hub from being
mistaken for informative agreement.  Applying the same calculation to
leave-self-out database pseudo queries gives $\bar I_a^d$; the calibration-set mean
is $\bar I_a^q$.  Their ratio
\begin{equation}
 C_a=\frac{\bar I_a^q}{\bar I_a^d+\epsilon}
 \label{eq:capacity-retention}
\end{equation}
measures how much reference-state information survives deployment shift.  Write $\operatorname{CI}_{.95}$ for a resampling-based 95\% interval. A
paired bootstrap~\citep{efron1979} estimates the 95\% interval
$\mathrm{CI}_C=\operatorname{CI}_{.95}[\log(C_S/C_H)]$; negative values favor
Hard and positive values favor Soft.

\paragraph{Independent relational evidence.}
Capacity measures how much information remains, but not whether the residual
geometry is coherent.  We therefore evaluate two views: the $K$ residual-block
norms (anchor-mass view $A$) and the complete normalized residual descriptor
(global-direction view $G$). In each view, we compare the calibration-image cosine
Gram matrix with the Gram matrix of the corresponding top-1 database matches.
The calibration-to-database score matrix also supplies the reciprocal rank of each
selected database column among calibration images and the fraction of reciprocal
top-1 matches.  A self-calibrated joint score combines non-negative centered
kernel alignment~\citep{cka2019}, mean reciprocal rank, mutual-top-1 support, and
effective-rank preservation, each normalized by its database-pseudo-query
reference.  Let $\kappa$, $M$, and $F$ denote these first three terms, and define
$u(x;x^d)=\operatorname{clip}(x/(x^d+\epsilon),0,1)$.  If
$r_q$ and $r_m$ denote the effective ranks of the centered query and
matched-descriptor Gram matrices, $P=\min(r_q/r_m,r_m/r_q)$ measures rank
preservation. A superscript $d$ denotes the corresponding database-pseudo-query
reference. The joint score is
\begin{equation}
 J=\bigl[u(\kappa;\kappa^d)u(M;M^d)u(F;F^d)\bigr]^{1/3}u(P;P^d).
 \label{eq:relation-joint}
\end{equation}
The geometric mean requires concurrent support from the three normalized
terms: a near-zero term limits the combined score. This multiplicative
construction expresses a conjunction of evidence, not statistical independence
or a calibrated probability of a correct match.
Effective rank is the participation ratio~\citep{participation2017}
$(\sum_j\eta_j)^2/(\sum_j\eta_j^2+\epsilon)$ of nonnegative Gram eigenvalues
$\eta_j$. In $J_a^R$, $a$ specifies assignment and $R$ the relational view.
For $R\in\{A,G\}$, repeated fixed calibration subsamples give
\begin{equation}
 m_R=\operatorname{median}(J_S^R-J_H^R),\qquad
 \mathrm{CI}_R=\operatorname{CI}_{.95}(J_S^R-J_H^R).
 \label{eq:relation-ci}
\end{equation}

Writing $b_R^-,b_R^+$ for the bounds of $\mathrm{CI}_R$ and $b_C^+$ for the upper
bound of $\mathrm{CI}_C$, the frozen assignment route is
\begin{equation}
\mathcal R_{\rm assign}=
\begin{cases}
H, & b_A^+<0\ \wedge\ b_G^+<0,\\
H, & b_C^+<0\ \wedge\ m_A<0\ \wedge\ m_G<0,\\
R, & b_A^->0,\\
S, & \text{otherwise}.
\end{cases}
\label{eq:assignment-route}
\end{equation}
The first two cases admit evidence deletion only under two-view structural
support or under significant capacity loss with coherent relation direction.
The route label $R$ denotes a reversible score mixture, distinct from the
relational-view index. Let $S_H,S_S$ be the Hard and Soft branch similarities,
$S_R$ their mixture, and $w_S$ its Soft weight. For
$\delta=\operatorname{clip}(m_A,-1,1)$,
\begin{equation}
 w_S=\sin^2\!\left(\frac{\pi(\delta+1)}{4}\right),\qquad
 S_R=(1-w_S)S_H+w_SS_S .
 \label{eq:assignment-angular}
\end{equation}
Zero is the only decision boundary.  Soft is the default because it does not
irreversibly discard non-winning anchor evidence.

\subsection{Support--agreement spectral carrier}
\label{sec:spectral-shaping}
For the selected assignment, a database-only float64 PCA of normalized VLAD
descriptors gives mean $\boldsymbol\mu$, directions $U$, and supported
eigenvalues $\lambda_j$.  A direction is retained when
$\lambda_j>\lambda_{\max}\max(N_d,d_v)\epsilon_{64}$, where
$\lambda_{\max}$ is the largest eigenvalue, $d_v$ is the descriptor dimension,
and $\epsilon_{64}$ is float64 machine precision. For exponent $\gamma$, let
\begin{equation}
 \phi_\gamma(\mathbf v)=\operatorname{L2}\!\left(
 \operatorname{diag}[(\lambda_j+\epsilon)^{-\gamma/2}]
 U^\top(\mathbf v-\boldsymbol\mu)\right).
 \label{eq:fixed-spectral-family}
\end{equation}
The conventional controls $\gamma\in\{0,0.5,1\}$ give centered PCA
projection, partial whitening, and full whitening~\citep{jegou2012negative}.
Original cosine retains the uncentered, unprojected descriptors.

For codebook $s$, let $\rho_s$ be the Pearson correlation between per-axis
database and calibration-set energies in the PCA basis. If $\{\nu_j\}$ are the
eigenvalues of the uncentered database Gram matrix, its participation ratio is
$\operatorname{PR}^{u}_s=(\sum_j\nu_j)^2/(\sum_j\nu_j^2+\epsilon)$.
The spectral-transfer score is
\begin{equation}
 \alpha_s=\operatorname{clip}(\rho_s,0,1)^{2\operatorname{PR}^{u}_s} .
\end{equation}
Both statistics are computed in the residual coordinates of the selected
assignment and vocabulary.
We aggregate construction dispersion using the normal-consistent MAD
scale~\citep{mad1993}. For $N_s=|\mathcal S|$ replicas, define
\begin{equation}
\begin{aligned}
 \widehat\sigma_{\rm MAD}(\alpha)
   &=\frac{\operatorname{MAD}_s(\alpha_s)}{\Phi^{-1}(3/4)},\\
 z_{.975}&=\Phi^{-1}(.975),\\
 c&=\operatorname{clip}\!\left(
 \operatorname{median}_s\alpha_s-
 \frac{z_{.975}\widehat\sigma_{\rm MAD}(\alpha)}{\sqrt{N_s}},0,1\right).
\end{aligned}
 \label{eq:spectral-confidence}
\end{equation}
Here $N_s=3$, and $c$ is a reliability score with a normal-reference scale.
Applying the same construction to $\rho_s$
gives the robust lower-side estimate $\rho_-$.  If
$\Delta_\mu$ is the median relative distance between calibration and database mean
descriptors, centering strength is
\begin{equation}
 \beta=\operatorname{clip}[\Delta_\mu(1-\rho_-),0,1] .
 \label{eq:center-confidence}
\end{equation}

Let $S_0$ be the original descriptor cosine, $S_c$ the cosine after centering
and projection, and $S_w(c)$ the cosine after additionally reweighting the
supported PCA coordinates by $\lambda_j^{-c/2}$.  With
$\ell(t)=\sin^2(\pi t/2)$, the local score is
\begin{equation}
 S_{\rm local}=(1-\ell(c))
 [(1-\ell(\beta))S_0+\ell(\beta)S_c]+\ell(c)S_w(c).
 \label{eq:angular-spectral}
\end{equation}
Each branch is independently L2-normalized, making
\cref{eq:angular-spectral} a direct-sum cosine kernel with the exact fallback
$S_{\rm local}=S_0$ at $c=\beta=0$.

Calibration compares three candidate carriers:
$S_A=S_{\rm local}$, fixed power-$0.5$ $S_P=S_w(1/2)$, and the original
no-PCA cosine $S_0$.  For carrier $m$, let $t_{m,s}(q)$ be the Top-1 database
index under codebook replica $s$.  Its mean database support and pairwise
replica agreement are
\begin{equation}
 \begin{aligned}
 U_m&=\frac1{N_s}\sum_{s\in\mathcal S}
 \left|\{t_{m,s}(q):q\in\mathcal Q\}\right|,\\
 A_m^{ij}&=\frac1{N_q}\sum_q
 \mathbf{1}\{t_{m,i}(q)=t_{m,j}(q)\}.
 \end{aligned}
 \label{eq:spectral-support-agreement}
\end{equation}
Here $i,j\in\mathcal S$ index distinct codebook replicas, not image patches.
Support alone can reward noisy coverage, whereas agreement alone admits
stable collapse.  The outer carrier selector therefore uses
\begin{equation}
 m^*=\begin{cases}
 0,&U_P\le U_0\ \wedge\
       \max_{i<j}(A_P^{ij}-A_0^{ij})<0,\\
 P,&U_P>U_A,\\
 P,&U_A>U_P\ \wedge\
       \max_{i<j}(A_A^{ij}-A_P^{ij})<0,\\
 A,&\text{otherwise}.
 \end{cases}
 \label{eq:spectral-carrier-route}
\end{equation}
Original cosine is restored when partial whitening fails to
expand support \emph{and} is unanimously less reproducible.  Power overrides
the adaptive carrier when it has a strict support advantage, or when adaptive
support expansion is unanimously less reproducible. Otherwise the inner
adaptive carrier is preserved, including exact support ties.

\subsection{CLS fusion with a retention veto}
\label{sec:cls-veto}
The local score captures anchored patch residuals, whereas CLS supplies global
semantic evidence.  Their score scales are calibrated from leave-self-out
database rows.  Let $\sigma_b^d$ be the median off-diagonal row standard
deviation for branch $b\in\{l,g\}$, where $l$ denotes local and $g$ denotes
CLS. Let $S_l$ be local similarity and $S_g$ cosine similarity of normalized
global tokens. Define $B_b$ as the smaller of mean calibration and pseudo-query
information for branch $b$, $J_b$ as its median calibrated relation score, and
$H_b$ as its pooled extreme-support coverage, specified below. Branch quality is
\begin{equation}
 Q_b=(B_bJ_bH_b)^{1/3},
 \label{eq:branch-quality}
\end{equation}
If $O_b$ database items are occupied by the
pooled top-1 predictions, $M_b$ is the largest observable support, and
$p_b^{\max}$ is the largest pooled occupancy probability, then
$H_b=\sqrt{(O_b/M_b)/(p_b^{\max}M_b)}$.  It penalizes both unused gallery
support and a dominant top-1 hub.
The cube root preserves the scale of the component scores while penalizing
a weak information, relation, or coverage component. This quality score
provides nonnegative relative weights for the candidate fusion:
\begin{equation}
 S_F=\frac{Q_l}{Q_l+Q_g}\frac{S_l}{\sigma_l^d}
     +\frac{Q_g}{Q_l+Q_g}\frac{S_g}{\sigma_g^d}.
 \label{eq:cls-fusion}
\end{equation}

Quality alone can still overestimate CLS under condition shift.  Let $T_b$ be
the ratio between the mean top-1 z-score of calibration rows and its
database-pseudo-query analogue.  We compare CLS with the local branch through
score-evidence retention $r_T=T_g/(T_l+\epsilon)$, information-capacity ratio
$r_C=B_g/(B_l+\epsilon)$, and information--relation reliability ratio
\begin{equation}
 r_J=\frac{\sqrt{B_gJ_g}}{\sqrt{B_lJ_l}+\epsilon}.
 \label{eq:cls-retention-ratios}
\end{equation}
The point test $Q_g\ge Q_l$
is retained only if $r_T\ge1$.  A complementarity rescue is allowed when
local and CLS row-dispersion statistics are negatively correlated for every
codebook and all three non-collapse ratios are at least one.  Specifically,
$S_l^s(q,j)$ is the local score for query $q$ and database image $j$ under
construction $s$, while $S_g(q,j)$ is the global-token score.
$\operatorname{Std}_j$ denotes row standard deviation over database images,
and $\operatorname{Corr}_q$ denotes Pearson correlation over calibration queries.
Thus $\operatorname{corr}_s=\operatorname{Corr}_q[
\operatorname{Std}_j S_l^s(q,j),\operatorname{Std}_j S_g(q,j)]$, and
\begin{equation}
 \begin{aligned}
 \operatorname{admit}_{\rm CLS}={}&
 (Q_g\ge Q_l\ \wedge\ r_T\ge1)\\
 &\vee\bigl(\forall s:\operatorname{corr}_s<0\ \wedge\
 r_C,r_T,r_J\ge1\bigr).
 \end{aligned}
 \label{eq:cls-veto}
\end{equation}
The CLS decision, local quality, and local score unit are frozen using the
inner adaptive carrier $S_A$ before applying
\cref{eq:spectral-carrier-route}. The selected carrier supplies the local
score in \cref{eq:cls-fusion}, using the fixed fusion weights and scales.
The final score is $S_F$ with that
substitution when CLS is admitted and exactly the selected local carrier
otherwise.
The boundaries at zero and one express directional support and retention
relative to the local branch.

\subsection{Calibration and deployment}
TFA-C64 first evaluates the residual hypotheses on $\mathcal D$ and the
64 calibration images, fixes assignment, and estimates spectral and CLS
controls. The resulting head is reused for every evaluation query without
re-estimating target statistics. Calibration requires neither place labels
nor model-weight updates.

\paragraph{Full-calibration reference.}
TFA-FullQ applies the same target-calibration formulation with
$\mathcal Q$ equal to the entire unlabeled evaluation query set, replacing
the budget $N_q=64$ by its full size. This idealized data-access setting
provides an empirical full-access reference for finite-sample calibration.
Independent deployment uses TFA-C64.

\paragraph{Larger-backbone system.}
The separate TFA-C64 system in Table~\ref{tab:iclr-strongest} combines
dual vocabularies with target-score calibration. It likewise fixes its
statistics on disjoint calibration images before processing evaluation
queries. Appendix~\ref{sec:native-interfaces} specifies its feature interface
and the complementary-fold full-calibration reference.

\section{Dataset sources and protocol attribution}
\label{sec:dataset-sources}
Table~\ref{tab:dataset-sources} attributes the datasets used in the main
comparisons and supplementary analyses. Dataset publications and benchmark
releases are distinguished from our retrieval adaptations; a source citation
does not imply that our split, gallery construction, or positive threshold
is identical to the source paper's localization experiment.

\begin{table}[ht]
\centering\small
\caption{Dataset sources. Multiple maps, conditions, heights, and retrieval
directions share their parent dataset reference.}
\label{tab:dataset-sources}
\setlength{\tabcolsep}{4pt}
\begin{tabular}{@{}p{.29\linewidth}p{.67\linewidth}@{}}\toprule
Dataset / protocols & Original source and benchmark attribution\\\midrule
Pitts30k / Pitts250k & Pittsburgh imagery~\citep{pittsburgh2013}; retrieval splits~\citep{netvlad}.\\
MSLS-val & Mapillary Street-Level Sequences~\citep{msls}.\\
St-Lucia & \citet{stlucia2008}; curated retrieval release~\citep{vprbenchmark2022}.\\
Eynsham & \citet{eynsham2009}; curated retrieval release~\citep{vprbenchmark2022}.\\
GSV-Brussels & Brussels subset of GSV-Cities~\citep{gsvcities}.\\
Essex3in1 & \citet{essex2018}; VPR-Bench release~\citep{vprbench2021}.\\
SPED & Specific Places Dataset~\citep{sped}.\\
Nordland & Seasonal railway traversals~\citep{nordland2013}.\\
AmsterTime & Historical/contemporary images~\citep{amstertime}.\\
CrossSeason & Cross-season correspondences~\citep{crossseason2019}; VPR-Bench release~\citep{vprbench2021}.\\
4Seasons (six maps) & Cross-season driving dataset~\citep{fourseasons}.\\
RobotCar (nine conditions) & Original data~\citep{robotcar2017}; Seasons benchmark~\citep{robotcarseasons}.\\
Baidu-Mall & Indoor localization dataset~\citep{baidumall2017}; AnyLoc packaging~\citep{anyloc}.\\
V4RL shopping & Shopping Street, urban place-recognition dataset~\citep{v4rl2018}.\\
SVOX (five conditions) & Street View Oxford~\citep{svox2021}.\\
UAV-VisLoc & UAV imagery and satellite maps~\citep{uavvisloc}.\\
VPAIR-full & Aerial retrieval/localization dataset~\citep{vpair}.\\
SUES-200 (four heights) & Multi-height drone/satellite benchmark~\citep{sues200}.\\
DenseUAV & Low-altitude UAV self-positioning dataset~\citep{denseuav}.\\
University (both directions) & University-1652~\citep{university1652}.\\
Urbanscape & CrossLoc benchmark~\citep{crossloc2022}; our rendered-reference retrieval protocol.\\
Park & UAV-to-synthetic-map localization environment from RIM~\citep{rim2026}, with near-nadir views and extensive low-texture regions; adapted for retrieval.\\
Tokyo24/7 & View-synthesis place-recognition benchmark~\citep{tokyo247}.\\
Nardo-Air & Aerial stress-test data distributed with AnyLoc~\citep{anyloc}.\\
AnyVisLoc & Low-altitude multi-view localization benchmark~\citep{anyvisloc2026}.\\
\bottomrule\end{tabular}
\end{table}

Park uses UAV observations and rendered map references from the
challenging, largely low-texture environment described in RIM~\citep{rim2026}.
Urbanscape retains the CrossLoc image source but uses our reference-map
retrieval construction. These two protocols use 50\,m horizontal-distance
positives. UAV-VisLoc uses coverage-qualified 100\,m recall. Other protocols
retain their registered positive sets. Public data formatting also draws on
the Deep Visual Geo-Localization Benchmark~\citep{vprbenchmark2022} and,
where noted, VPR-Bench~\citep{vprbench2021} and AnyLoc~\citep{anyloc}.

\section{Evaluation protocols and system configurations}
\label{sec:iclr-evidence-scope}
Tables~\ref{tab:iclr-primary} and~\ref{tab:iclr-aerial} evaluate the
database-calibrated TFA head; disjoint target calibration is reported as
TFA-C64. Full-query measurements provide a transductive reference.
Pittsburgh contributes two related protocols; six
4Seasons maps are grouped for dataset-level counts. RobotCar uses a
nine-condition macro and SUES a four-height macro. Pitts250k scores its full
gallery while fitting PCA on a fixed 5,000-image sample. Park and Urbanscape
AnyLoc controls use seed 0, except Urbanscape DINOv3, which averages three
seeds. Primary protocols require at least 500 queries; smaller
protocols are used for diagnostics. Retrieval positives follow the released
or registered dataset protocol, not one shared geometric radius.

\subsection{Native-system interfaces}
\label{sec:native-interfaces}
We distinguish three configurations. TFA is the single-vocabulary DB-only
head. TFA-C64 and TFA-FullQ use the target-aware controller at finite and
full query access, respectively. TFA-C$N$ denotes a calibration budget of $N$.
Table~\ref{tab:iclr-strongest} uses the dual-vocabulary G-based system, whereas
Table~\ref{tab:iclr-primary} uses the single-vocabulary B-based head. The G-based extension combines
a source K32 vocabulary and a map-fitted K64 vocabulary with score calibration.
The three map-codebook seeds are experimental replicas, not three fused branches.

Table~\ref{tab:iclr-strongest} uses each method's native interface.
SelaVPR++ uses the DINOv2-L performance model at $322^2$, with 512-bit
Hamming Top-100 retrieval followed by 4096-D global-feature reranking;
Table~\ref{tab:iclr-primary} instead uses its DINOv2-B single-branch model.
TF-VPR uses $336^2$. AnyLoc uses released K32 vocabularies at $640\times480$;
SegVLAD-PreT retains native image sizes and segment voting.
RIA retains paper values on Pitts30k and St-Lucia; other Table~\ref{tab:iclr-strongest}
entries use our covariance implementation with Giant L31 value features at
$640\times480$, averaged over projection seeds 0/1/42.
Sample4Geo uses one University-trained ConvNeXt-B checkpoint at $384^2$ across
all aerial protocols. Its published drone-to-satellite R@1 is 92.65; AP is
not substituted for R@5/R@10. Author-log and checkpoint-inference values are
marked separately in Table~\ref{tab:iclr-aerial}.

In Table~\ref{tab:iclr-strongest}, TFA-FullQ (Cross-fit) uses even/odd folds:
each query uses statistics from the opposite fold. Fixed-budget TFA-C$N$ uses disjoint calibration
and evaluation sets; evaluation complements vary with $N$. SVOX uses
condition-specific calibration. Cross-system rankings compare reported
performance under these declared access settings.

\paragraph{University reverse protocol.}
University-1652 satellite-to-drone uses 701 queries and 51,355 database images.
Predictions are sealed before identity-based evaluation. TFA/AnyLoc average
seeds 0/1/2; RIA uses projection seeds 0/1/42; TF-VPR is deterministic.

\paragraph{Dual-vocabulary calibration.}
The controlled D2-B batch study uses database-fitted K32 and K64 branches, so
its gain does not require a released external vocabulary. It fixes partial
whitening, equal branch weights, unit temperature, and even/odd folds. It
improves the stronger single-capacity branch in 27/28 cells, but exceeds the
TFA-FullQ head in 19/23 protocols. RobotCar is a major exception:
even/odd calibration gives 71.94 against 89.40 for K64, while contiguous halves
give 90.88. Nordland has the opposite split sensitivity, 65.47 versus 56.76.
These controls show the sensitivity of score calibration to calibration-set
composition.

The G-backbone Cal64 curves use ten disjoint splits per calibration size and
three codebooks averaged within each split. Source-like development maps for
this score operator are Pitts30k, St-Lucia, and VPAIR; SPED and SVOX use the
frozen operator. The reported 97.41--99.76\% absolute-performance retention is
different from net-gain recovery, which is 85.1/50.3/21.9/80.3/89.6\% for
Pitts30k/SPED/St-Lucia/VPAIR/SVOX. The St-Lucia gain is only 0.30 points, making
its gain-recovery ratio particularly sensitive to small changes.

\paragraph{Scale controls and timing scope.}
At fixed final-token $322^2$ interfaces, the inner-head comparison gives
four-map means 73.12/73.17/77.00 for DINOv2-B/L/G and margins 4.00/0.46/1.97
to the strongest matched head. These measurements isolate the inner head,
with the final outer selector excluded.
Nordland-L is the largest failure,6.79 points below its strongest comparator.
Encoder-only batch-one RTX 3090 latencies are 12.47/33.05/100.35 ms for B/L/G;
these figures do not measure end-to-end retrieval or map preparation.

\subsection{Deployment timing}
\label{sec:deployment-timing}
The RTX 3090 batch-one comparison uses FP32, CUDA-resident query tokens,
20 warm-up iterations, and 100 timed batches, with exact Top-10 search over
the 1,231-image AmsterTime gallery. Image decoding, resizing, transfers, and
map preparation are excluded for all methods. The common DINOv2-B encoder
takes 12.468 ms and the TFA-FullQ head takes 1.223 ms, totaling 13.691 ms.
This configuration stores 245.62 MiB of database descriptors and 223.88 MiB
of projection bases.

Table~\ref{tab:deployment-h100} measures the released database-only and C64
implementations on one H100 PCIe GPU using identical DINOv2-B/14 Pitts30k
features at $322^2$, a 10,000-image gallery, and the same 6,752 evaluation
queries. The remaining 64 queries calibrate C64 and are excluded from
evaluation for both variants. Three fresh-process repetitions alternate
variant execution order, with four CPU threads and a 20-GiB CUDA allocator
limit. Setup includes database PCA and controller estimation; shared feature
extraction and codebook construction are excluded. Retrieval timings include
descriptor loading, transfers, score computation, and Top-10 selection for
all three codebook replicas. Replicas are evaluated independently, rather
than fused into an ensemble. The implementations use their native numerical
and batching paths; these wall times characterize the released pipelines,
including their I/O and intermediate computations. All repetitions produce
identical rankings within each variant.

\begin{table}[ht]
\centering\small
\caption{Pitts30k deployment costs on H100, mean$\pm$sample SD over three
repetitions. Retrieval covers 6,752 queries and three codebook replicas.
Amortized time is per query per replica, rather than batch-one latency.}
\label{tab:deployment-h100}
\begin{tabular}{lrrr}\toprule
System & Setup (s) & Retrieval (s) & Amortized (ms)\\\midrule
TFA & 78.38$\pm$1.78 & 7.01$\pm$0.17 & 0.346$\pm$0.008\\
TFA-C64 & 138.52$\pm$2.83 & 27.22$\pm$1.02 & 1.344$\pm$0.050\\\bottomrule
\end{tabular}
\end{table}
Within C64 setup, target-aware controller estimation takes
37.93$\pm$0.35 s. The additional cost reflects both calibration and
implementation-specific computation; it does not isolate the cost of
accessing target observations alone.

\subsection{Component analysis}
\label{sec:component-analysis}
The spectral controls in Table~\ref{tab:iclr-spectral} isolate transformations
under the TFA-FullQ assignment. Its last three columns compare complete
database-only, C64, and TFA-FullQ heads. The component ablations below use
query-calibrated control and are distinct from that complete-head comparison.

\paragraph{Unlabeled spectral selection.}
Table~\ref{tab:iclr-spectral}(a) holds the no-label assignment fixed before changing
the spectral operation. Whitening helps SPED but collapses AmsterTime; the
same support--agreement rule keeps adaptive on SPED and partial whitening on
Pitts30k/AmsterTime. The remaining SPED gap to post-hoc whitening is 0.33 points.
On ten carrier configurations, the final head wins/ties/loses 3/7/0 against
fixed partial whitening, with an equal-cell mean gain of 0.82 points. Against
the earlier adaptive complete head the corresponding count is 7/2/1, mean gain
2.22, and worst change$-0.43$. The two comparisons quantify gains over a
fixed spectrum and an adaptive head, respectively.

\paragraph{Both support and agreement are needed.}
A support tie incorrectly treated as evidence for partial whitening reduces
MSLS-val from 75.86 to 73.29. A support-contraction-only fallback instead sends
Tokyo24/7 to original cosine at 87.09, whereas requiring agreement to contradict
the intervention gives 92.27. Tokyo serves as a small-query mechanism case.
Together these controls demonstrate the complementary roles of retrieval
support and cross-construction agreement.

\paragraph{Spectral selection transfers across physical maps.}
The four maps in Table~\ref{tab:iclr-spectral}(b) were excluded from formulation of
the outer rule. Its predictions were frozen before the registered geometry was
read. The rule selects partial/adaptive/original/original, matching the best
displayed candidate on all four maps and improving the adaptive-spectrum mean by 0.81.
This evaluates outer-selector transfer within a dataset family previously
used for inner-head development.

\paragraph{Complementary roles of assignment and CLS.}
Stable multi-anchor assignment on the Nardo-D3 stress case obtains 46.48 R@1,
whereas single-anchor assignment reaches 68.54. The converse occurs on a
4Seasons City Loop split, where forcing single-anchor assignment loses 2.41
points. In the four-cell City-Loop/Countryside D2/D3 assignment comparison, the
capacity--relation rule selects the post-hoc best fixed assignment in three
cells and remains within 0.56 points in the fourth.

Across 17 development cells, unconditional CLS fusion has mean gain 0.64 but
degrades 8 cells and loses up to 7.51 points. The retention veto has mean
gain 2.17, admits a positive change in four cells, and returns exactly to the
local branch in thirteen. Selective admission therefore improves the benefit
of CLS while preserving the local branch when evidence is weak.
The assignment, spectrum, and CLS experiments measure component-specific effects
on their respective evaluation populations.

\paragraph{Reproduction and supplementary coverage.}
AnyLoc's author-pipeline Pitts30k check gives 87.63/94.69 R@1/R@5 against the
reported 87.7/94.7. The RIA reconstruction gives 83.13/91.05, below its
fixed-seed 86.36 R@1 reference; it remains explicitly a diagnostic port. Its
positive results are nevertheless retained in the University comparison.
AnyVisLoc contains 24 scenes. The supplementary SegVLAD adaptation covers
seven scenes, so no all-scene aggregate is reported for that adaptation.

\paragraph{Interpreting the mechanism.}
The confidence statistic in the inner spectral kernel depends on residual
coordinates. For SPED-G, retaining native raw center norms changes spectral
participation ratios and can make the adaptive gate under-intervene even when
fixed whitening remains useful. The result motivates stating the declared
unit-center interface precisely, rather than claiming invariance to all
backbone facets and residual parameterizations. The direct-sum proposition
establishes exact geometric fallback only; empirical recall improvements come
from the ablations and held-out decisions, not from an unsupervised accuracy
guarantee.

\begin{table}[t]\centering\footnotesize
\setlength{\tabcolsep}{3pt}
\caption{Ground-backbone comparison, R@1 (\%). TFA uses DB-only calibration; C64 uses disjoint calibration (mean$\pm$SD). TFA-FullQ is the transductive reference.}
\label{tab:ground-d3-audit}
\begin{tabular}{@{}llrrrrrr@{}}\toprule
Dataset & Backbone & AnyLoc & TF-VPR & RIA$^d$ & TFA & TFA-C64 & TFA-FullQ\\\midrule
MSLS-val & DINOv2-B & 52.70 & 41.49 & 44.41 & 70.09 & 71.16$\pm$2.55 & 75.86\\
MSLS-val & DINOv3-B & 47.39 & 40.54 & 39.77 & 72.21 & 63.17$\pm$5.24 & 71.53\\
\addlinespace
SPED & DINOv2-B & 75.57 & 73.48 & 69.85 & 85.12 & 83.79$\pm$0.63 & 85.01\\
SPED & DINOv3-B & 70.95 & 76.94 & 69.85 & 87.92 & 85.75$\pm$0.84 & 86.55\\
\addlinespace
Nordland & DINOv2-B & 29.50 & 29.46 & 21.03 & 39.60 & 30.23$\pm$3.27 & 35.76\\
Nordland & DINOv3-B & 25.15 & 30.81 & 19.07 & 40.17 & 31.51$\pm$4.90 & 30.98\\
\addlinespace
Essex3in1 & DINOv2-B & 78.25 & 73.81 & 71.11 & 82.38 & 82.21$\pm$0.85 & 83.81\\
Essex3in1 & DINOv3-B & 76.35 & 66.19 & 68.57 & 79.68 & 78.68$\pm$2.54 & 81.27\\
\addlinespace
CrossSeason & DINOv2-B & 99.83 & 100.00 & 99.65 & 100.00 & 99.97$\pm$0.08 & 100.00\\
CrossSeason & DINOv3-B & 100.00 & 100.00 & 99.83 & 100.00 & 100.00$\pm$0.00 & 100.00\\
\addlinespace
V4RL shopping & DINOv2-B & 89.05 & 82.97 & 86.97 & 69.54 & 86.91$\pm$2.66 & 87.29\\
V4RL shopping & DINOv3-B & 91.53 & 88.49 & 84.57 & 72.18 & 89.36$\pm$2.27 & 87.13\\
\addlinespace
SVOX-5 macro & DINOv3-B & 50.09 & 56.51 & 44.50 & 68.00 & 57.94$\pm$7.28 & 79.40\\
\addlinespace
\bottomrule\end{tabular}
\par\smallskip\raggedright\scriptsize C64 excludes calibration images; other columns evaluate the full query set. SVOX uses 64 calibration images across five conditions, with macro averaging before split-wise SD. TFA-FullQ is a system-level reference, not a paired calibration-budget ablation.
\end{table}

\clearpage
\section{Additional calibration results}
\begin{table}[!ht]
\centering\footnotesize
\setlength{\tabcolsep}{3pt}
\caption{Ground query-calibrated variants, R@1 (\%). DINOv2-B/14, $322^2$. C64: mean$\pm$SD; DB-only results appear in Table~\ref{tab:iclr-primary}.}
\label{tab:full-query-reference}
\begin{tabular}{@{}lrr@{\hspace{12pt}}lrr@{}}\toprule
Dataset & TFA-C64 & TFA-FullQ & Dataset & TFA-C64 & TFA-FullQ\\\midrule
Pitts30k & 84.33$\pm$1.59 & 86.00 & CrossSeason & 99.97$\pm$0.08 & 100.00\\
Pitts250k & 83.70$\pm$1.63 & 87.51 & 4S Business Campus & 95.55$\pm$0.65 & 94.64\\
MSLS-val & 71.16$\pm$2.55 & 75.86 & 4S City Loop & 82.88$\pm$0.77 & 83.94\\
St-Lucia & 95.72$\pm$2.87 & 97.95 & 4S Countryside & 36.76$\pm$5.75 & 41.85\\
Eynsham & 85.19$\pm$3.92 & 88.45 & 4S Office Loop & 90.20$\pm$0.75 & 90.54\\
GSV-Brussels & 84.75$\pm$0.90 & 84.36 & 4S Old Town & 70.20$\pm$1.87 & 70.71\\
Essex3in1 & 82.21$\pm$0.85 & 83.81 & 4S Parking Garage & 99.30$\pm$0.28 & 99.51\\
SPED & 83.79$\pm$0.63 & 85.01 & RobotCar (nine cond.) & 88.33$\pm$0.46 & 89.34\\
Nordland & 30.23$\pm$3.27 & 35.76 & Baidu-Mall & 63.55$\pm$1.71 & 62.29\\
AmsterTime & 52.00$\pm$1.76 & 52.72 & V4RL shopping & 86.91$\pm$2.66 & 87.29\\
\bottomrule\end{tabular}
\par\smallskip\raggedright\scriptsize 4S: 4Seasons. TFA-FullQ uses all queries for calibration and evaluation; C64 evaluates the disjoint complement. TFA-FullQ provides a system-level reference; Table~\ref{tab:iclr-strongest} compares calibration budgets within a fixed system.
\end{table}

\begin{table}[!ht]
\centering\footnotesize
\setlength{\tabcolsep}{3pt}
\caption{Aerial query-calibrated variants, R@1/5/10 (\%). C64: disjoint 64-query calibration (mean$\pm$SD); TFA-FullQ: full-query calibration. DB-only results appear in Table~\ref{tab:iclr-aerial}.}
\label{tab:iclr-aerial-access}
\begin{tabular}{@{}llrrrrrr@{}}\toprule
 & & \multicolumn{3}{c}{TFA-C64} & \multicolumn{3}{c}{TFA-FullQ}\\
\cmidrule(lr){3-5}\cmidrule(lr){6-8}
Dataset & Backbone & R@1 & R@5 & R@10 & R@1 & R@5 & R@10\\\midrule
UAV-VisLoc & D2-B & 15.05$\pm$2.34 & 31.54$\pm$3.89 & 38.92$\pm$4.00 & 18.50 & 37.46 & 44.94\\
 & D3-B & 18.33$\pm$2.22 & 36.43$\pm$3.55 & 43.90$\pm$3.59 & 20.61 & 40.42 & 47.96\\
\rowcolor{gray!10} VPAIR-full & D2-B & 38.52$\pm$2.93 & 57.56$\pm$3.48 & 65.71$\pm$3.53 & 39.71 & 60.05 & 68.63\\
\rowcolor{gray!10} & D3-B & 38.71$\pm$1.73 & 57.42$\pm$1.99 & 65.69$\pm$1.98 & 40.32 & 60.52 & 69.30\\
SUES-200 & D2-B & 58.65$\pm$2.00 & 82.58$\pm$1.26 & 89.72$\pm$0.87 & 60.89 & 83.77 & 90.45\\
 & D3-B & 66.55$\pm$2.26 & 86.06$\pm$1.58 & 91.38$\pm$1.27 & 73.58 & 91.59 & 96.00\\
\rowcolor{gray!10} DenseUAV & D2-B & 7.53$\pm$0.82 & 23.12$\pm$2.41 & 34.21$\pm$3.53 & 7.85 & 24.45 & 36.09\\
\rowcolor{gray!10} & D3-B & 6.22$\pm$1.24 & 20.56$\pm$3.52 & 30.59$\pm$4.28 & 8.57 & 27.64 & 39.34\\
Park & D2-B & 36.61$\pm$1.33 & 74.77$\pm$2.05 & 85.58$\pm$1.81 & 38.97 & 78.08 & 88.64\\
 & D3-B & 44.87$\pm$0.58 & 80.52$\pm$0.81 & 88.41$\pm$0.98 & 43.18 & 80.57 & 88.99\\
\rowcolor{gray!10} Urbanscape & D2-B & 81.97$\pm$3.81 & 94.21$\pm$1.79 & 96.20$\pm$1.03 & 84.93 & 95.78 & 97.32\\
\rowcolor{gray!10} & D3-B & 84.71$\pm$1.93 & 94.38$\pm$0.74 & 96.04$\pm$0.51 & 87.11 & 95.10 & 96.33\\
University D$\to$S & D2-B & 40.39$\pm$4.40 & 58.23$\pm$5.23 & 65.39$\pm$5.25 & 44.27 & 61.81 & 68.78\\
 & D3-B & 45.28$\pm$2.59 & 68.72$\pm$2.94 & 76.78$\pm$2.76 & 45.74 & 69.59 & 77.55\\
\rowcolor{gray!10} University S$\to$D & D2-B & 53.73$\pm$4.67 & 66.58$\pm$4.60 & 71.84$\pm$5.35 & 57.44 & 69.09 & 75.13\\
\rowcolor{gray!10} & D3-B & 85.59$\pm$0.37 & 93.24$\pm$0.25 & 94.99$\pm$0.18 & 85.78 & 93.20 & 95.01\\
\bottomrule\end{tabular}
\par\smallskip\raggedright\scriptsize D2/D3-B: DINOv2/v3-B; D/S: drone/satellite. SUES uses 64 calibration images across four heights, with macro averaging before split-wise SD. University directions are separate protocols. C64 excludes calibration images; TFA-FullQ evaluates all queries. Metrics follow Table~\ref{tab:iclr-aerial}. DB-only and query-calibrated variants differ in both control statistics and data access.
\end{table}

\clearpage
\section{Calibration-size sensitivity}
\label{sec:paired-single-calibration}
\paragraph{Paired single-vocabulary calibration.}
With identical DINOv2-B/14 features at $322^2$ and the same target-aware
controller, TFA-C64 and TFA-FullQ are evaluated on the same held-out C64
complement in each of ten splits. TFA-FullQ additionally accesses these
evaluation images during calibration. Mean R@1 for C64 versus TFA-FullQ is
84.58/86.35 on Pitts30k, 83.79/84.64 on SPED, 95.72/98.04 on St-Lucia,
51.99/52.41 on AmsterTime, and 63.55/62.37 on Baidu-Mall.
C64 sample standard deviations are 1.39, 0.63, 2.87, 1.75, and 1.71 points,
respectively. The paired results quantify the effect of calibration access
within a fixed system, with no uniform advantage from using all queries.
\begin{figure}[h]
\centering
\includegraphics[width=\linewidth]{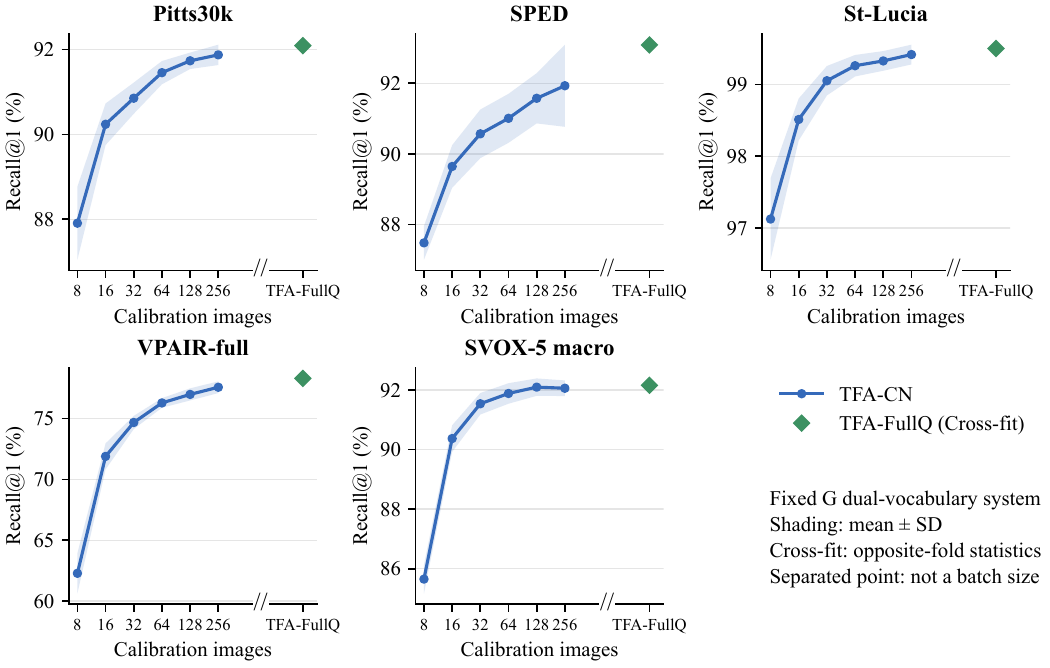}
\caption{Disjoint calibration-size experiment, R@1. Blue: DINOv2-G
dual-vocabulary TFA-C$N$ system, mean$\pm$SD over ten splits. The separated TFA-FullQ (Cross-fit) point
uses opposite-fold query statistics (Table~\ref{tab:iclr-strongest}); its horizontal
position is schematic. All points use the same backbone and dual-vocabulary
configuration; calibration and evaluation sets follow the stated protocols.}
\label{fig:iclr-calibration}
\end{figure}

\clearpage
\section{Extended cross-method retrieval comparisons}
\begin{figure}[h]
\centering
\includegraphics[width=.96\linewidth,height=.81\textheight,keepaspectratio]{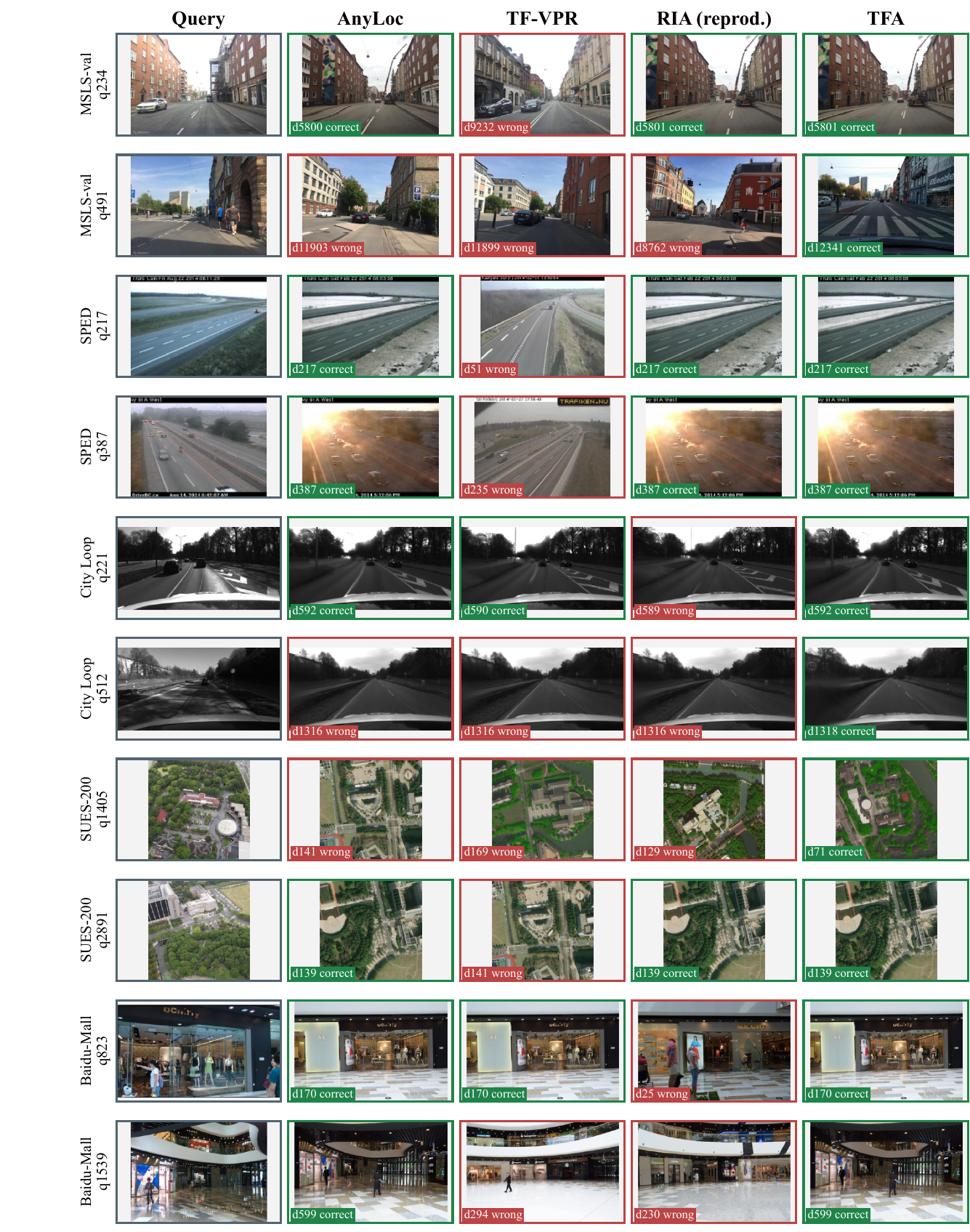}
\caption{Additional Top-1 retrievals where DB-only TFA succeeds and at least
one baseline fails (DINOv2-B, seed 0). Green/red indicate correctness under
the dataset positives; RIA uses our same-feature reconstruction.
Two query-index-tertile examples per dataset exclude the main-figure queries.
Images retain their full field of view.}
\label{fig:supp-retrieval-success}
\end{figure}
\clearpage
\begin{figure}[h]
\centering
\includegraphics[width=.96\linewidth,height=.81\textheight,keepaspectratio]{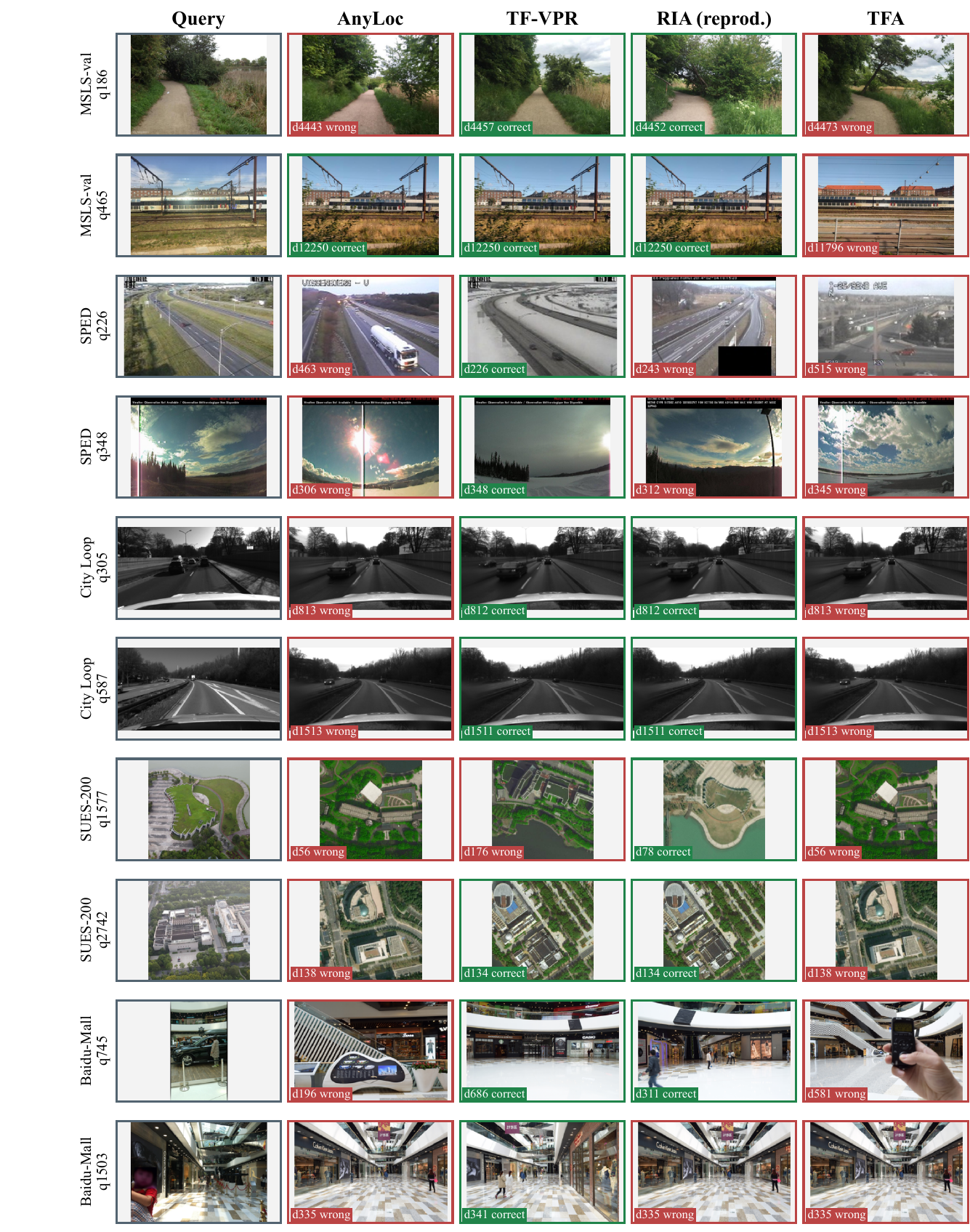}
\caption{Complementary failures: DB-only TFA misses a match recovered by at
least one baseline. The protocol and sampling rule follow
Figure~\ref{fig:supp-retrieval-success}, with two examples per dataset.
Correctness follows the dataset positives, not visual similarity alone.}
\label{fig:supp-retrieval-failure}
\end{figure}

\end{document}